\pdfoutput=1
\documentclass{article} 
\usepackage{hyperref}
\usepackage{times}
\usepackage{geometry}

\usepackage{mkolar_definitions}
\usepackage{url}
\usepackage{authblk}

\usepackage{graphicx}
\usepackage{subfigure}
\usepackage{multirow}
\usepackage{floatrow}
\newfloatcommand{capbtabbox}{table}[][\FBwidth]

\usepackage{natbib}

\usepackage{algorithm}
\usepackage{algorithmic}

\usepackage{enumerate}

\usepackage{amsthm,amsmath,amssymb}

\newenvironment{packed_enum}{
\begin{enumerate}
\setlength{\itemsep}{1pt}
\setlength{\parskip}{0pt}
\setlength{\parsep}{0pt}
}{\end{enumerate}}

\begin{document} 

\title{On the Power of Adaptivity in Matrix Completion and Approximation}

\author[1]{
Akshay Krishnamurthy
\thanks{akshaykr@cs.cmu.edu}}
\author[2]{
Aarti Singh
\thanks{aarti@cs.cmu.edu}}

\affil[1]{Computer Science Department\\
Carnegie Mellon University}
\affil[2]{Machine Learning Department\\
Carnegie Mellon University}

\maketitle

\begin{abstract}
We consider the related tasks of matrix completion and matrix approximation from missing data and propose adaptive sampling procedures for both problems. 
We show that adaptive sampling allows one to eliminate standard incoherence assumptions on the matrix row space that are necessary for passive sampling procedures. 
For exact recovery of a low-rank matrix, our algorithm judiciously selects a few columns to observe in full and, with few additional measurements, projects the remaining columns onto their span. 
This algorithm exactly recovers an $n \times n$ rank $r$ matrix using $O(n r \mu_0 \log^2(r))$ observations, where $\mu_0$ is a coherence parameter on the column space of the matrix. 
In addition to completely eliminating any row space assumptions that have pervaded the literature, this algorithm enjoys a better sample complexity than any existing matrix completion algorithm.
To certify that this improvement is due to adaptive sampling, we establish that row space coherence is necessary for passive sampling algorithms to achieve non-trivial sample complexity bounds.

For constructing a low-rank approximation to a high-rank input matrix, we propose a simple algorithm that thresholds the singular values of a zero-filled version of the input matrix.
The algorithm computes an approximation that is nearly as good as the best rank-$r$ approximation using $O(nr \mu \log^2(n))$ samples, where $\mu$ is a slightly different coherence parameter on the matrix columns. 
Again we eliminate assumptions on the row space.
\end{abstract}

\section{Introduction}
\label{sec:intro}

While the cost of data acquisition has decreased significantly across the spectrum of scientific applications, it has failed to keep up with the increasing complexity of the systems and processes being studied. 
As a concrete example, routing optimization in communication networks or personalization in social networks involve making inferences at the granularity of individual nodes and links, and these tasks become more challenging as the networks grow.
In other words, both the amount of data generated by the network \emph{and} the complexity of statistical problems scale with the size of the network.
This phenomenon is prevalent across modern statistical applications and the result is that analysts face the challenge of making meaningful inferences from extremely undersampled datasets.

A number of innovative algorithmic ideas, including the compressive sensing paradigm \citep{candes2008introduction}, show that one can reliabily perform inference in the presence of undersampling, but many of these methods are limited in that the algorithms do not \emph{interact} with the acquisition mechanism. 
Adaptive sampling, where the inference algorithm interacts with the data measurement process, is a promising approach to tolerate more severe undersampling.
Indeed this is true in several settings, where we are now aware of adaptive sampling methods that outperform all passive schemes~\citep{haupt2011distilled}.

This paper proposes adaptive sampling algorithms for low rank matrix completion and matrix approximation. 
In the completion problem, we would like to exactly recover a low rank matrix after observing only a small fraction of its entries.
In the approximation problem, rather than exact recovery, we aim to find a low rank matrix that approximates, in a precise sense, the input matrix, which need not be low rank. 
In both problems, we are only allowed to observe a small number of matrix entries, although these entries can be chosen sequentially and in a feedback-driven manner.

The thesis of our work is that adaptive sampling allows us to remove incoherence assumptions that have pervaded the literature. 
Previous analyses show that if the energy of the matrix is spread out fairly uniformly across its coordinates, then passive uniform-at-random samples suffice for completion or approximation. 
In contrast, our work shows that adaptive sampling algorithms can focus measurements appropriately to solve these problems even if the energy is non-uniformly distributed.
Handling non-uniformity is essential in a variety of problems involving outliers, for example network monitoring problems with anomalous hosts, or recommendation problems with popular items or highly active users. 
This is a setting where passive algorithms fail, as we will show.

We make the following contributions:

\begin{enumerate}
\item For the matrix completion problem, we give a simple algorithm that exactly recovers an $n \times n$ rank $r$ matrix using at most $O(nr \mu_0 \log^2(r))$ measurements where $\mu_0$ is the coherence parameter on the column space of the matrix. 
This algorithm outperforms all existing results on matrix completion both in terms of sample complexity and in the fact that we place no assumptions on the row space of the matrix.
The algorithm is extremely simple, runs in $O(nr^2)$ time, and can be implemented in one pass over the columns of the matrix.
\item We complement this sufficient condition with a lower bound showing that in the absence of row-space incoherence, \emph{any} passive scheme must see $\Omega(n^2)$ entries.
This concretely demonstrates the power of adaptivity in the matrix completion problem.
\item For matrix approximation, we analyze an algorithm that, after an adaptive sampling phase, approximates the input matrix by the top $r$ ranks of an appropriately rescaled zero-filled version of the matrix. 
We show that with just $O(nr \mu \log^2(n))$ samples, this approximation is competitive with the best rank $r$ approximation of the matrix.
Here $\mu$ is a coherence parameter on each column of the matrix; as before we make no assumptions about the row space of the input. 
Again, this result significantly outperforms existing results on matrix approximation from passively collected samples.
\end{enumerate}

This paper is organized as follows: we conclude this introduction with some basic definitions and then turn to related work in Section~\ref{sec:related}.
The main results, consequences, and more detailed comparisons to existing work are given in Section~\ref{sec:results}.
In Section~\ref{sec:experiments}, we provide some simulation that validate our theoretical results. 
We provide proofs in Section~\ref{sec:proofs} and conclude with some future directions in Section~\ref{sec:discussion}.
Some details are deferred to the appendix. 



\subsection{Preliminaries}
Before proceeding, let us set up some notation used throughout the paper.
We are interested in recovering, or approximating, a $d \times n$ matrix $X$ from a set of at most $M$ observations (We assume $d \le n$.).
We denote the columns of $X$ by $x_1, \ldots, x_n \in \RR^d$ and use $t$ to index the columns. 
We use $x_t(i)$ to denote the $i$th coordinate of the column vector $x_t$. 

We will frequently work with the truncated singular value decomposition (SVD) of $X$ which is given by zero-ing out its smaller singular values.
Specifically, write $X = U_r \Sigma_r V_r^T + U_{-r}\Sigma_{-r} V_{-r}^T$ where $[U_r, U_{-r}]$ (respectively $[V_r, V_{-r}]$) forms an orthonormal matrix and $\Sigma_r = \textrm{diag}(\sigma_1, \ldots, \sigma_r), \Sigma_{-r} = \textrm{diag}(\sigma_{r+1}, \ldots, \sigma_{d})$ are diagonal matrices with $\sigma_1 \ge \ldots \ge \sigma_r \ge \sigma_{r+1} \ge \ldots \ge \sigma_d$.
The truncated singular value decomposition is $X_r = U_r \Sigma_r V_r^T$, which is the best rank-$r$ approximation to $X$ both in Frobenius and spectral norm.

We also use capital letters to denote subspaces.
We often overload notation by using the same symbol to refer to a subspace and any orthonormal basis for that subspace.
Specifically, if $U \subset \RR^d$ is a subspace with dimension $r$, we may sometimes use $U$ to refer to a $d\times r$ matrix whose columns are an orthonormal basis for that subspace and vice versa. 
We use $U^\perp$ to denote the orthogonal complement to the subspace $U$ and $\Pcal_U$ to refer to the orthogonal projection operator onto $U$. 

As we are dealing with missing data and sampling, we also need some notation for subsampling operations. 
Let $[d]$ denote the set $\{1, \ldots, d\}$ and let $\Omega$ be a list of $m$ values from $[d]$, possibly with duplicates (One can think of $\Omega$ as a vector in $[d]^m$ and $\Omega(j)$ is the $j$th coordinate of this vector).
If $x \in \RR^d$, then $x_\Omega \in \RR^m$ is the vector formed putting $x(i)$ in the $j$th coordinate if $\Omega(j) = i$ and $\Rcal_{\Omega}x$ is a zero-filled rescaled version of $x$ with $\Rcal_{\Omega}x(i) = 0$ if $i \notin \Omega$ and $\Rcal_{\Omega}x(i) = dx(i)/|\Omega|$ if $i \in \Omega$.
In other words, $\Rcal_{\Omega}$ is a $d\times d$ diagonal matrix with the $(i,i)$th entry equal to $d/|\Omega|$ if $i \in \Omega$ and zero otherwise. 

For a $r$-dimensional subspace $U\subset \RR^d$, $U_{\Omega} \in \RR^{m \times r}$ is a \emph{matrix} formed by doing a similar subsampling operation to the \emph{rows} of any orthonormal basis for the subspace $U$, e.g. the $j$th row of $U_{\Omega}$ is the $i$th row of $U$ if $\Omega(j) = i$.
Note that $U_{\Omega}$, and even the span of the columns of $U_{\Omega}$, may not be uniquely defined, as they both depend on the choice of basis for $U$.
Nevertheless, we will use $\Pcal_{U_{\Omega}}$ to denote the projection onto the span of any single set of columns constructed by this subsampling operation.

In the matrix completion problem, where we aim for exact recovery, we require that $X$ has rank at most $r$, meaning that $\sigma_{r+1} = \ldots = \sigma_n = 0$. 
Thus $X = X_r$, and our goal is to recover $X_r$ exactly from a subset of entries.
Specifically, we focus on the $0/1$ loss; given an estimator $\hat{X}$ for $X$, we would like to bound the probability of error:
\begin{align}
\label{eq:exact_risk}
R_{01}(\hat{X}) \triangleq \PP\left( \hat{X} \ne X \right).
\end{align}

In the approximation problem, we relax the low rank assumption but are only interested in approximating the action of $X_r$.
The goal is to find a rank $r$ matrix $\hat{X}$ that minimizes:
\[
R(\hat{X}) = \|X - \hat{X}\|_F.
\]
The matrix $X_r$ is the global minimizer (subject to the rank-$r$ constraint), and our task is to approximate this low rank matrix effectively. 
Specifically, we will be interested in finding matrices $\hat{X}$ that satisfy excess risk bounds of the form:
\begin{align}
\label{eq:approx_risk}
R(\hat{X}) \triangleq \|X - \hat{X}\|_F \le \|X - X_r\|_F + \epsilon \|X\|_F
\end{align}
Rescaling the excess risk term by $\|X\|_F$ is a form of normalization that has been used before in the matrix approximation literature~\citep{frieze2004fast,drineas2006fastii,drineas2006fastiii,rudelson2007sampling}.
While bounds of the form $(1+\epsilon)\|X - X_r\|_F$ may seem more appropriate when the bottom ranks are viewed as noise term, achieving such a bound seems to require highly accurate approximations of the SVD of the input matrix~\citep{drineas2008relative}, which is not possible given the extremely limited number of observations in our setting.
Equation~\ref{eq:approx_risk} can be interpreted by dividing by $\|X\|_F$, which shows that $\hat{X}$ captures almost as large a fraction of the energy of $X$ as $X_r$ does.

Apart from the observation budget $M$ and the approximation rank $r$, the other main quantity governing the difficulty of these problems is the subspace coherence parameter.
For a $r$ dimensional subspace $U$ of $\RR^d$, define
\[
\mu(U) = \frac{d}{r} \max_{i \in [d]} \|\Pcal_U e_i\|_2^2,
\]
which is a standard measure of subspace coherence~\citep{recht2011simpler}).
The quantity $\mu_0 \triangleq \mu(U_r)$, which is bounded between $1$ and $d/r$, measures how correlated the principal column space of the matrix $X$ is with any single standard basis element.
When this maximal correlation is small, the energy of the matrix is spread out fairly uniformly across the rows of the matrix, although it can be non-uniformly distributed across the columns. 
Without loss of generality we use the matrix column-space coherence $\mu_0$ instead of the row-space analog, and we will see that the parameter $\mu_0$ controls the sample complexity of our adaptive procedure. 

In classical results on matrix completion, the parameter $\mu_0' \triangleq \max\{\mu(U_r), \mu(V_r)\}$ instead governs the sample complexity.
When $\mu_0'$ is small, both principal subspaces are incoherent, so that the energy of the matrix is uniformly spread across the entries. 
Informally, this means that a random sample of entries captures the salient features of the matrix, and, indeed, the number of uniform-at-random samples necessary and sufficient for exact recovery scales linearly with $\mu_0'$~\citep{recht2011simpler}.

Such an incoherence assumption does not translate appropriately to the approximate recovery problem, since the matrix is no longer low rank, but some measure of uniformity is still necessary. 
One one hand, the statistics literature typically assumes that the matrix $X$ can be decomposed into an incoherent low rank matrix and a stochastic perturbation~\citep{negahban2012restricted,koltchinskii2011nuclear}.
On the other hand, classical results on matrix approximation make no stochastic assumptions, but also do not need uniformity, as they do not consider the missing data setting~\citep{frieze2004fast,drineas2006fastii}. 
As we aim to bridge these two lines of research by considering matrix approximation with missing data, we remove the stochastic assumption.
We instead turn to an alternative assumption to ensure that the high ranks of the matrix are well-behaved under sampling. 


We parameterize the problem by a quantity related to the usual definition of incoherence:
\[
\mu = \max_{t \in [n]} \frac{d||x_t||_{\infty}^2}{||x||_2^2},
\]
which is the maximal column coherence. 
Here, we make no stochastic assumptions, but notice that this is a restriction on the higher ranks of the matrix.
We also make no assumptions about the row space of the matrix\footnote{As before this could equivalently be the column space with assumption on the maximal row coherence. Without loss of generality we parametrize columns and column spaces throughout this paper.}. 

\section{Related Work}
\label{sec:related}


The literature on low rank matrix approximation is extremely vast and we do not attempt to cover all of the existing ideas. 
Instead, we focus on the most relevant lines of work to our specific problems. 
We briefly mention some related work on adaptive sensing.

For matrix completion, a series of papers provide better and better analyses of the nuclear norm minimization procedure, finally showing that $M \asymp (n+d)r \mu_0' \log(n)$ uniform-at-random observations are sufficient to exactly recover a rank $r$ matrix with high probability~\citep{candes2010power, candes2009exact, gross2011recovering, recht2011simpler, chen2013incoherence}.
These results involve the parameter $\mu_0'$, implying that both matrix subspaces must be incoherent for strong guarantees.
Recently, two papers have relaxed the row-space incoherence assumption with adaptive sampling~\citep{krishnamurthy2013low,chen2014coherent}.
Our analysis leads to a better sample complexity than both of these results.
We defer a more detailed discussion to after Theorem~\ref{thm:exact_ub}.


A number of authors have studied matrix completion with noise and under weaker assumptions.
The most prominent difference between our work and all of these is a relaxation of the main incoherence assumptions. 
Both \citet{candes2010matrix}, and \citet{keshavan2010matrix} require that both the row and column space of the matrix of interest is highly incoherent.
\citet{negahban2012restricted} instead use a notion of \emph{spikiness}, but that too places assumptions on the row space of interest.
\citet{koltchinskii2011nuclear} consider matrices with bounded entries, which is related to the spikiness assumption.
In comparison, our results make essentially no assumptions about the row space, leading to substantially more generality.
This is the thesis of our work; one can eliminate row space assumptions in matrix recovery problems through adaptive sampling. 


Another close line of work is on matrix sparsification~\citep{achlioptas2007fast,arora2006fast,achlioptas2013near}.
Here, the goal is to zero out a large number of entries of a given matrix while preserving global properties such as the principal subspace.
The main difference from the matrix completion literature is that the entire matrix is observed, and this allows one to relax incoherence assumptions.
The only result from this literature that does not require knowledge of the matrix is a random sampling scheme of~\citet{achlioptas2007fast}, but it is only competitive with matrix completion results when the input matrix has entries of fairly constant magnitude~\cite{koltchinskii2011nuclear}.
Interestingly, this requirement is essentially the same as the spikiness assumption of \citet{negahban2012restricted} and the bounded magnitude assumption of \citet{koltchinskii2011nuclear}.

Several techniques have been proposed for matrix approximation in the fully observed setting, optimizing computational complexity or other objectives. 
A particularly relevant series of papers is on the column subset selection (CSS) problem, where the span of several judiciously chosen columns is used as to approximate the principal subspace. 
One of the best approaches involves sampling columns according to the statistical leverage scores, which are the norms of the rows of the $n \times r$ matrix formed by the top $r$ right singular vectors~\citep{boutsidis2009improved,boutsidis2011near,drineas2008relative}. 
Unfortunately, this strategy does not seem to apply in the missing data setting, as the distribution used to sample columns -- which are subsequently used to approximate the matrix -- depends on the unobserved input matrix.
Approximating this distribution seems to require a very accurate estimate of the matrix itself, and this initial estimate would suffice for the matrix approximation problem. 
This difficulty also arises with volume sampling~\citep{guruswami2012optimal}, another popular approach to CSS; the sampling distribution depends on the input matrix and we are not aware of strategies for approximating this distribution in the missing data setting.

In terms of adaptive sampling, a number of methods for recovery of sparse, possibly structured, signals have been shown to outperform passive methods~\citep{haupt2011distilled,malloy2011sequential,tanczos2013adaptive,balakrishnan2012recovering,krishnamurthy2013recovering}.
While having their share of differences, these methods can all be viewed as either binary search or local search methods, that iteratively discard irrelevant coordinates and focus measurements on the remainder. 
In particular, these methods rely heavily on the sparsity and structure of the input signal, and extensions to other settings have been elusive. 
While a low rank matrix is sparse in its eigenbasis, the search-style techniques from the signal processing community do not seem to leverage this structure effectively and these approaches do not appear to be applicable to our setting.

Some of these adaptive sampling efforts focus specifically on recovering or approximating highly structured matrices, which is closely related to our setting. 
\cite{tanczos2013adaptive} and~\cite{balakrishnan2012recovering} consider variants of biclustering, which is equivalent to recovering a rank-one binary matrix from noisy observations. 
\cite{singh2012completion} recover noisy ultrametric matrices while~\cite{krishnamurthy2012efficient} uses a similar idea to find hierarchical clustering from adaptively sampled similarities.
All of these results can be viewed as matrix completion or approximation, but impose significantly more structure on the target matrix than we do here.
For this reason, many of these algorithmic ideas also do not appear to be useful in our setting. 


\section{Results}
\label{sec:results}
In this section we develop the main theoretical contributions of this manuscript. 
We first turn to the matrix completion problem, where we improve the results of~\citet{krishnamurthy2013low} and show that $O(dr + nr \mu_0 \log^2r)$ samples suffice to recover a rank $r$ matrix whose column space has coherence bounded by $\mu_0$. 
We complement this result with some necessary conditions on passive and adaptive matrix completion algorithms.
Then, we turn to the low rank approximation problem, where we describe a simple algorithm and show that it achieve the excess risk bound in Equation~\ref{eq:approx_risk} with $O(nr\mu\log^2(n)/\epsilon^4)$ samples. 
We also provide a detailed comparison of this result with prior work.

\subsection{Matrix Completion} 
\label{sec:exact}

\begin{algorithm}[t]
\caption{Adaptive Matrix Completion $(X \in \RR^{d \times n}, m)$}
\begin{packed_enum}
\item{} Let $U = \emptyset$.
\item{} Randomly draw entries $\Omega \subset [d]$ of size $m$ uniformly with replacement.
\item{} For each column $x_t$ of $X$ ($t \in [N]$):
\begin{packed_enum}
\item{} If $||x_{t\Omega} - \Pcal_{U_{\Omega}} x_{t\Omega}||_2^2 > 0$:
\begin{packed_enum}
\item{} Fully observe $x_t$ and add to $U$ (orthogonalize $U$). 
\item{} Randomly draw a new set $\Omega$ of size $m$ uniformly with replacement. 
\end{packed_enum}
\item{} Otherwise $\hat{x}_t \gets U(U_{\Omega}^TU_{\Omega})^{-1}U_{\Omega} x_{t\Omega}$.
\end{packed_enum}
\item{} Return $\hat{X}$ with columns $\hat{x}_t$.
\end{packed_enum}
\label{alg:exact_mc}
\end{algorithm}

Our algorithm for the matrix completion problem is identical to the algorithm of \citet{krishnamurthy2013low}.
The procedure, whose pseudocode is displayed in Algorithm~\ref{alg:exact_mc}, streams the columns of the matrix $X$ into memory and iteratively adds directions to an estimate for the column space of $X$.
The algorithm maintains a subspace $U$ and, when processing the $t$th column $x_t$, estimates the norm of $\Pcal_{U^\perp} x_t$ using only a few entries of $x_t$.
We will ensure that, with high probability, this estimate will be non-zero if and only if $x_t$ contains a new direction.
If the estimate is non-zero, the algorithm asks for the remaining entries of $x_t$ and adds the new direction to the subspace $U$.
Otherwise, $x_t$ lies in $U$ and we will see that the algorithm already has sufficient information to complete the column $x_t$.

Therefore, the key ingredient of the algorithm is the estimator for the projection onto the orthogonal complement of the subspace $U$. 
This quantity is estimated as follows.
Using a list of $m$ locations $\Omega$ from $[d]$, we downsample both $x_t$ and an orthonormal basis $U$ to $x_{t\Omega}$ and $U_{\Omega}$. 
We then use $\| x_{t \Omega} - \Pcal_{U_\Omega} x_{t \Omega}\|^2$ as our estimate. 
It is easy to see that this estimator leads to a test with one-sided error, since the estimator is identically zero if $x_t \in U$.
In our analysis, we establish a relative-error deviation bound, which allows us to control the error of our test for energy outside of $U$.

A subtle but critical aspect of the algorithm is the choice of $\Omega$.
The list $\Omega$ always has $m$ elements, and each element is sampled uniformly with replacement from $[d]$.
More importantly, we only resample $\Omega$ when we add a direction to $U$.
This ensures that the algorithm does not employ too much randomness, which would lead to an undesireable logarithmic dependence on $n$. 

The analysis of this test statistic and the reconstruction procedure leads to the following guarantee on the performance of the algorithm, whose proof is deferred to Section~\ref{sec:proofs}.
\begin{theorem}
\label{thm:exact_ub}
Let $X \in \RR^{d \times n}$ be a matrix of rank $r$ whose column space $U$ has coherence $\mu(U) \le \mu_0$. 
Then the output of Algorithm~\ref{alg:exact_mc} has risk:
\begin{align}
R_{01}(\hat{X}) \le 10r^2 \exp\left\{-\sqrt{\frac{m}{32r \mu_0}}\right\}
\end{align}
provided that $m \ge 4r \mu_0\log(2r/\delta)$.
Equivalently, whenever $m \ge 32r \mu_0 \log^2(10r^2/\delta)$, we have $R_{01}(\hat{X}) \le \delta$.
The sample complexity is $dr + nm$ and the running time is $O(nmr + r^3m + dr^2)$.
\end{theorem}

To the best of our knowledge, this result provides the strongest guarantee for the matrix completion problem. 
The vast majority of results require both incoherent row and column spaces and are therefore considerably more restrictive than ours~\citep{candes2010power,candes2009exact,gross2011recovering,recht2011simpler,chen2013incoherence}. 
For example, Recht shows that by solving the nuclear norm minimization program, one can recover $X$ exactly, provided that the number of measurements exceeds $32(d+n)r \max\{\mu_0', \mu_1^2\}\log^2(n)$ where recall that $\mu_0'$ upper bounds the coherence of both the row and column space, and $\mu_1$ provides another incoherence-type assumption (which can be removed~\citep{chen2013incoherence}).  
Our result improves on his not only in relaxing the row space incoherence assumption, but also in terms of sample complexity, as we remove the logarithmic dependence on problem dimension. 

As another example, \citet{gittens2011spectral} showed that Nystrom method can recover a rank $r$ matrix from randomly sampling $O(r \log r)$ columns. 
While his result matches ours in terms of sample complexity, he analyzes positive-semidefinite matrices with incoherent principal subspace, which translates to assuming that both row and column spaces are incoherent. 
Again, in relaxing this assumption, our result is substantially more general.

We mention two papers that allow coherent row spaces.
The first is the paper of \citet{krishnamurthy2013low}, that gives a weaker analysis of Algorithm~\ref{alg:exact_mc} resulting in a polynomially worse dependence on $r$. 
The other is the two-phase algorithm of \citet{chen2014coherent} based on local coherence sampling. 
Their algorithm requires $O((n+d)r \mu_0 \log(n))$ samples which is weaker than our guarantee in that it has a slightly super-linear dependence on problem dimension.
An interesting consequence of Theorem~\ref{thm:exact_ub} is that the amortized number of samples per column is \emph{completely independent} of the problem dimension.

Regarding computational considerations, the algorithm operates in one pass over the columns, and need only store the matrix in condensed form, which requires $O((n+d)r)$ space. 
Specifically, the algorithm maintains a (partial) basis for column space and the coefficients for representing each column by that basis, which leads to an optimally condensed representation.
Moreover, the computational complexity of the algorithm is \emph{linear} in the matrix dimensions $d,n$ with mild polynomial dependence on the rank $r$. 
For this run-time analysis, we work in a computational model where accessing any entry of the matrix is a constant-time operation, which allows us to circumvent the $\Omega(dn)$ time it would otherwise take to read the input. 
In comparison, the two standard algorithms for matrix completion, the iterative Singular Value Thresholding Algorithm~\citep{cai2010singular} and alternating least-squares~\citep{jain2013low,hardt2013understanding}, are significantly slower than Algorithm~\ref{alg:exact_mc}, not only due to their iterative nature, but also in per-iteration running time.

\subsection{Necessary Conditions for Matrix Completion}
We now establish a lower bound on any passive sampling algorithm for the matrix completion problem.
Our lower bound shows that if the matrix has coherent row space, then any passive sampling scheme followed by any recovery algorithm requires $\sim dn$ samples. 

To formalize our lower bound we fix a sampling budget $M$ and consider an estimator to be a sampling distribution $q$ over $\{(i,j) | i \in [d], j \in [n]\}^M$ and a (possibly randomized) function $f : \{(\Omega, X_{\Omega})\} \rightarrow \RR^{d \times n}$ that maps a set of indices and values to a $d \times n$ matrix. 
Let $\Qcal$ denote the set of all such sampling distributions and let $\Fcal$ denote the set of all such estimators.
Lastly let $\Xcal$ denote the set of all $d \times n$ rank $r$ matrices with column incoherence at most $\mu_0$. 
We consider the minimax probability of error:
\begin{align*}
R^\star = \inf_{f \in \Fcal} \inf_{q \in \Qcal} \sup_{X \in \Xcal} \PP_{\Omega \sim q}\left[ f(\Omega, X_{\Omega} \ne X\right]
\end{align*}
where the probability also accounts for potential randomness in the estimator $f$. 
Note that since we make no assumptions about the distribution $q$ other than excluding adaptive distributions, this setup subsumes essentially all passive sampling strategies including uniform-at-random, deterministic, and distributions sampling entire columns.
The one exception is the bernoulli sampling model, where each entry $(i,j)$ is observed with probability $q_{ij}$ independently of all other entries, although we believe a similar lower bound holds there. 

The following theorem lower bounds success probability of any passive strategy and consequently gives a necessary condition on the sample complexity.
\begin{theorem}
\label{thm:exact_lb}
The minimax risk $R^\star$ satisfies:
\begin{align}
R^\star \ge \frac{1}{2} - \left\lceil\frac{m}{(1- \frac{r-1}{r \mu_0})d} \right\rceil\frac{1}{2(n-r)},
\end{align}
which approaches $1/2$ whenever:
\begin{align}
m = o\left((dn - dr)(1 + \frac{1}{r \mu_0} - \frac{1}{\mu_0})\right).
\end{align}
\end{theorem}

As a concrete instantiation of the theorem, if $\mu_0$ is bounded from below by any constant $c > 1$ (which is possible whenever $r \le d/c$), then the bound approaches $1/2$ whenever $m = o(d(n-r))$. 
Thus all passive algorithms must have sample complexity that is quadratic in the problem dimension.
In contrast, Theorem~\ref{thm:exact_ub} ensures that Algorithm~\ref{alg:exact_mc} has nearly linear sample complexity, which is a significant improvement over passive algorithms.

The literature contains several other necessary conditions on the sample complexity for matrix completion. 
A fairly simple argument shows that without any form of incoherence, one requires $\Omega(dn)$ samples to recover even a rank one matrix that is non-zero in just one entry.
This argument applies to both passive and adaptive sampling strategies and shows that some measure of incoherence is necessary.
With both row and column incoherence, but still under uniform sampling, \citet{candes2010power} prove that $\Omega(\mu_0' n r \log(n))$ observations are necessary to recover a $n \times n$ matrix. 

One can relax the incoherence assumption by non-uniform passive sampling, although the sampling distribution is matrix-specific as it depends on the local coherence structure~\citep{chen2014coherent}.
Unfortunately, one cannot compute the appropriate sampling distribution, before taking any measurements.
Our result shows that in the absence of row-space incoherence, there is no universal passive sampling scheme that can achieve a non-trivial sample complexity. 
Thus adaptive sampling is necessary to relax the incoherence assumption while retaining near-optimal sample complexity.

Finally, a parameter counting argument shows that even adaptive sampling requires $\Omega((d+n)r)$ samples. 
Each entry of a rank $r$ matrix can be expressed as a polynomial of the left and right singular vectors and the singular values, so the observations lead to a polynomial system in $(d+n)r+r$ variables. 
If $M < (d+n)r - r^2$ (there are $r(r+1)$ orthonormality constraints), then this system is underdetermined, and since it has one solution, it must have infinitely many, so that recovery is impossible. 
Consequently, our algorithm is nearly optimal, and significantly outperforms any passive sampling strategy. 

\subsection{Low Rank Approximation}
\label{sec:approx}
\begin{algorithm}[t]
\caption{Low Rank Approximation $(X, m_1, m_2)$}
\begin{packed_enum}
\item Pass 1: For each column, observe $\Omega_t$ of size $m_1$ uniformly at random with replacement and estimate $\hat{c}_t = \frac{d}{m_1}||x_{t,\Omega_t}||_2^2$. Estimate $\hat{f} = \sum_t \hat{c}_t$. 
\item Pass 2: Set $\tilde{X} = 0 \in \RR^{d \times n}$. 
\begin{packed_enum}
\item For each column $x_t$, sample $m_{2,t} = m_2 n \hat{c}_2/\hat{f}$ observations $\Omega_{2,t}$ uniformly at random with replacement.
\item Update $\tilde{X} = \tilde{X} + (\Rcal_{\Omega_{2,t}}x_t)e_t^T$. 
\end{packed_enum}
\item Compute the SVD of $\tilde{X}$ and output $\hat{X}$ which is formed by the top-$r$ ranks of $\tilde{X}$. 
\end{packed_enum}
\label{alg:approx}
\end{algorithm}

For the matrix approximation problem, we propose an adaptive sampling algorithm to obtain a low-rank approximation to $X$. 
The algorithm (see Algorithm~\ref{alg:approx} for pseudocode) makes two passes through the columns of the matrix. 
In the first pass, it subsamples each column uniformly at random and estimates each column norm and the matrix Frobenius norm. 
In the second pass, the algorithm samples additional observations from each column, and for each $t$, places the rescaled zero-filled vector $\Rcal_{\Omega_{2,t}}x_t$ into the $t$th column of a new matrix $\tilde{X}$, which is a preliminary estimate of the input, $X$.
Once the initial estimate $\tilde{X}$ is computed, the algorithm zeros out all but the top $r$ ranks of $\tilde{X}$ to form $\hat{X}$.
We will show that $\hat{X}$ has low excess risk, when compared with the best rank-$r$ approximation, $X_r$. 

A crucial feature of the second pass is that the number of samples per column is proportional to the squared norm of that column.
Of course this sampling strategy is only possible if the column norms are known, motivating the first pass of the algorithm, where we estimate precisely this sampling distribution. 
This feature allows the algorithm to tolerate highly non-uniform column norms, as it focuses measurements on high-energy columns, and leads to significantly better approximation.
This idea has been used before, although only in the exactly low-rank case~\citep{chen2014coherent}.


For the main performance guarantee, we only assume that the matrix has incoherent columns, that is $ d \|x_t\|_{\infty}^2/\|x_t\|_2^2 \le \mu$ for each column $x_t$. 
In particular we make no additional assumptions about the high-rank structure of the matrix.
We have the following theorem:
\begin{theorem}
\label{thm:approx_adv}
Set $m_1 \ge 32\mu \log(n/\delta)$ and assume $n \ge d$ and that $X$ has $\mu$-incoherent columns. 
With probability $\ge 1 - 2\delta$, Algorithm~\ref{alg:approx} computes an approximation $\hat{X}$ such that:
\begin{align*}
\|X - \hat{X}\|_F \le \|X - X_r\|_F + \|X\|_F\left(6\sqrt{\frac{r \mu}{m_2}}\log\left(\frac{d+n}{\delta}\right) + \left(6\sqrt{\frac{r \mu}{m_2}}\log\left(\frac{d+n}{\delta}\right)\right)^{1/2}\right)
\end{align*}
using $n(m_1+m_2)$ samples. 
In other words, the output $\hat{X}$ satisfies $\|X - \hat{X}\|_F \le \|X - X_r\|_F + \epsilon \|X\|_F$ with probability $\ge 1-2\delta$ and with sample complexity:
\begin{align}
32n\mu\log(n/\delta) + \frac{576}{\epsilon^4} n r \mu \log^2\left(\frac{d+n}{\delta}\right).
\label{eq:sample_comp}
\end{align}
\end{theorem}

The proof is deferred to Section~\ref{sec:proofs}.
The theorem shows that the matrix $\hat{X}$ serves as nearly as good an approximation to $X$ as $X_r$.
Specifically, with $O(nr\mu\log^2(d+n))$ observations, one can compute a suitable approximation to $X$. 
The running time of the algorithm is dominated by the cost of computing the truncated SVD, which is at most $O(d^2n)$.

While the dependence between the number of samples and the problem parameters $n,r$, and $\mu$ is quite mild and matches existing matrix completion results, the dependence on the error $\epsilon$ in Equation~\ref{eq:sample_comp} seems undesirable.
This dependence arises from our translation of a bound on $\|\tilde{X} - X\|_2$ into a bound on $\|\hat{X} - X\|_F$, which results in the $m_2^{-1/4}$-dependence in the error bound.
We are not aware of better results in the general setting, but a number of tighter translations are possible under various assumptions, and these can result in better guarantees. 
We mention just two such results here. 

\begin{proposition}
\label{cor:approx_exact_cor}
Under the same assumptions as Theorem~\ref{thm:approx_adv}, suppose further that $X$ has rank at most $r$. 
Then with probability $\ge 1- 2\delta$:
\begin{align*}
\|X - \hat{X}\|_F \le 20 \|X\|_F \sqrt{\frac{r\mu}{m_2}}\log\left(\frac{d+n}{\delta}\right)
\end{align*}
\end{proposition}
This proposition tempers the dependence on the error $\epsilon$ from $1/\epsilon^4$ to $1/\epsilon^2$ in the event that the input matrix has rank at most $r$. 
This gives a relative error guarantee for Algorithm~\ref{alg:approx} on the matrix completion problem, which improves on the one implied by Theorem~\ref{thm:approx_adv}. 
Note that this guarantee is weaker than Theorem~\ref{thm:exact_ub}, but Algorithm~\ref{alg:approx} is much more robust to relaxations of the low rank assumption as demonstrated in Theorem~\ref{thm:approx_adv}.

A similarly mild dependence on $\epsilon$ can be derived under the assumption that $X = A+R$, $A$ has rank $r$ and $R$ is some perturbation, which has the flavor of existing noisy matrix completion results. 
Here, it is natural to recover the parameter $A$ rather than the top $r$ ranks of $X$ and we have the following parameter recovery guarantee for Algorithm~\ref{alg:approx}:
\begin{proposition}
\label{cor:approx_stoch}
Let $X = A + R$ where $A$ has rank at most $r$.
Suppose further that $X$ has $\mu$-incoherent columns and set $m_1 \ge 32\mu \log(n/\delta)$.
Then with probability $\ge 1 - 2\delta$:
\begin{align}
\|\hat{X} - A\|_F \le 20 \sqrt{\frac{r\mu}{m_2}}\log\left(\frac{d+n}{\delta}\right)\left(\|A\|_F + \|R_{\Omega}\|_F\right) + \sqrt{8r}\|R\|_2
\end{align}
where the number of samples is $n(m_1+m_2)$ and $\Omega$ is the set of all entries observed over the course of the algorithm. 
\end{proposition}

To interpret this bound, let $\|A\|_F = 1$, and let $R$ be a random matrix whose entries are independently drawn from a Gaussian distribution with variance $\sigma^2/(dn)$. 
Note that this normalization for the variance is appropriate in the high-dimensional setting where $n,d \rightarrow \infty$, since we keep the signal-to-noise ratio $\|A\|^2_F/\|R\|^2_F = 1/\sigma^2$ constant. 
In this setting, the last term can essentially be ignored, since by the standard bound on the spectral norm of a Gaussian matrix, $\|R\|_2 = O(\sigma \sqrt{\frac{1}{d}}\log((n+d)/\delta))$ which will be lower order~\citep{achlioptas2007fast}. 
We can also bound $\|R_{\Omega}\|_F \le O(\sigma\sqrt{\frac{m_1+m_2}{d}}\log((n+d)/\delta))$ using a Gaussian tail bound. 
With $m_1 \le m_2$ we arrive at:
\begin{align*}
\|\hat{X} - A\|_F \le c_\star \left(\sqrt{\frac{r\mu}{m_2}} + \sigma \sqrt{\frac{r\mu}{d}}\right) \log^2\left(\frac{d+n}{\delta}\right),
\end{align*}
where $c_\star$ is some positive constant. 
In the high dimensional setting, when $r\mu = \tilde{o}(d)$, this shows that Algorithm~\ref{alg:approx} consistently recovers $A$ as long as $m_2= \tilde{\omega}(r \mu)$.
This second condition implies that the total number of samples uses is $\tilde{\omega}(nr\mu)$\footnote{The notation $\tilde{o}(\cdot), \tilde{\omega}(\cdot)$ is the Bachmann-Landau asymptotic notation but suppressing logarithmic factors.}.

\subsubsection{Comparison with Matrix Completion Results}
The closest result to Theorem~\ref{thm:approx_adv} is the result of \citet{koltchinskii2011nuclear} who consider a soft-thresholding procedure and bound the approximation error in squared-Frobenius norm.
They assume that the matrix has bounded entrywise $\ell_{\infty}$ norm and give an entrywise squared-error guarantee of the form:
\begin{align}
\|\hat{X} - X\|_F^2 \le \|X - X_r\|_F^2 + c dn\|X\|_{\infty}^2 \frac{n r\log(d+n)}{M}
\end{align}
where $M$ is the total number of samples and $c$ is a constant. 
Their bound is quite similar to ours in the relationship between the number of samples and the target rank $r$. 
However, since $dn\|X\|_{\infty}^2 \ge \|X\|_F^2$, their bound is significantly worse in the event that the energy of the matrix is concentrated on a few columns.

To make this concrete, fix $\|X\|_F = 1$ and let us compare the matrix where every entry is $\frac{1}{\sqrt{dn}}$ with the matrix where one column has all entries equal to $\frac{1}{\sqrt{d}}$. 
In the former, the error term in the squared-Frobenius error bound of Koltchinskii et al. is $nr \log(d+n)/M$ while our bound on Frobenius error is, modulo logarithmic factors, the square root of this quantity.
In this example, the two results are essentially equivalent.
For the second matrix, their bound deteriorates significantly to $n^2r \log(d+n)/M$ while our bound remains the same.
Thus our algorithm is particularly suited to handle matrices with non-uniform column norms. 

Apart from adaptive sampling, the difference between our procedure and the algorithm of \citet{koltchinskii2011nuclear} is a matter of soft- versus hard-thresholding of the singular values of the zero-filled matrix. 
In the setting of Proposition~\ref{cor:approx_stoch}, soft thresholding seems more appropriate, as the choice of regularization parameter allows one to trade off the amount of signal and noise captured in $\hat{X}$. 
While in practice one could replace the hard thresholding step with soft thresholding in our algorithm, there are some caveats with the theoretical analysis.
First, soft-thresholding does not ensure that $\hat{X}$ will be at most rank $r$, so it is not suitable for the matrix approximation problem.
Second, the resulting error guarantee depends on the sampling distribution, which cannot be translated to the Frobenius norm unless the distribution is quite uniform~\citep{negahban2012restricted,koltchinskii2011nuclear}. 
Thus the soft-thresholding procedure does not give a Frobenius-norm error guarantee in the non-uniform setting that we are most interested in.


The majority of other results on low rank matrix completion focus on parameter recovery rather than approximation~\citep{negahban2012restricted,keshavan2010matrix,candes2010matrix}.
It is therefore best to compare with Proposition~\ref{cor:approx_stoch}, where we show that Algorithm~\ref{alg:approx} consistently recovers the parameter, $A$. 
These results exhibit similar dependence between the number of samples and the problem parameters $n, r, \epsilon$ but hold under different notions of uniformity, such as spikiness, boundedness, or incoherence. 
Our result agrees with these existing results but holds under a much weaker notion of uniformity.

Lastly, we emphasize the effect of adaptive sampling in our bound. 
We do not need \emph{any} uniformity assumption over the columns of the input matrix $X$. 
All existing works on noisy low rank matrix completion or matrix approximation from missing data have some assumption of this form, be it incoherence~\citep{keshavan2010matrix,candes2010matrix}, spikiness~\citep{negahban2012restricted}, or bounded $\ell_{\infty}$ norm~\citep{koltchinskii2011nuclear}. 
The detailed comparison with the result of Koltchinskii et al. gives a precise characterization of this effect and shows that in the absence of such uniformity, our adaptive sampling algorithm enjoys a significantly lower sample complexity.

In the event of uniformity, our algorithm performs similarly to existing ones.
Specifically, we obtain the same relationship between the total number of samples $M$, the problem dimensions $n,d$ and the target rank $r$. 
If we knew \emph{a priori} that the matrix had near-uniform column lengths, we could simply omit the first pass of the algorithm, sample uniformly in the second pass and avoid the need for any adaptivity.

\section{Simulations}
\label{sec:experiments}

\begin{figure}
\subfigure[]{
\includegraphics[scale=0.19]{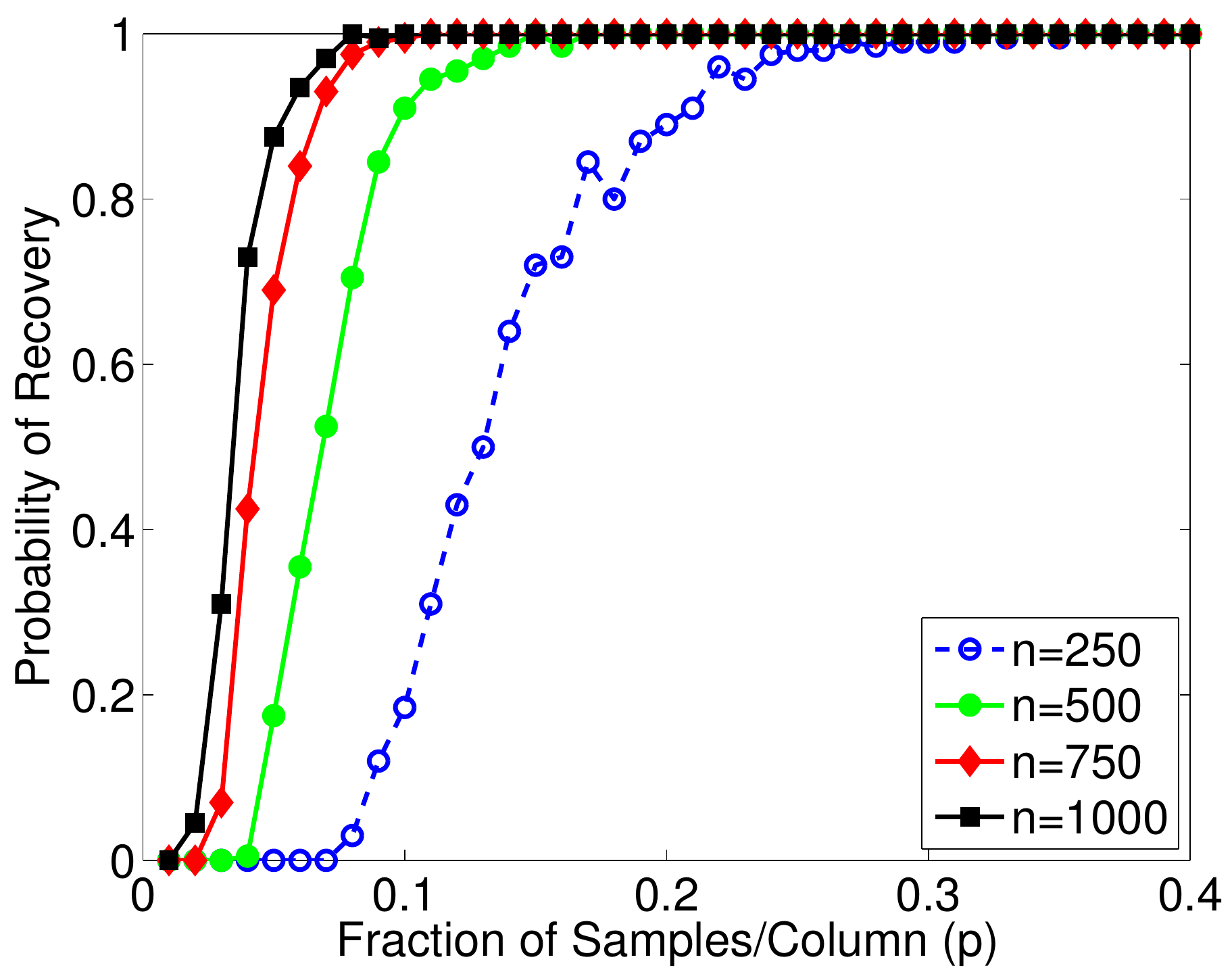}
\label{fig:exact_p_threshold_1}
}\subfigure[]{
\includegraphics[scale=0.19]{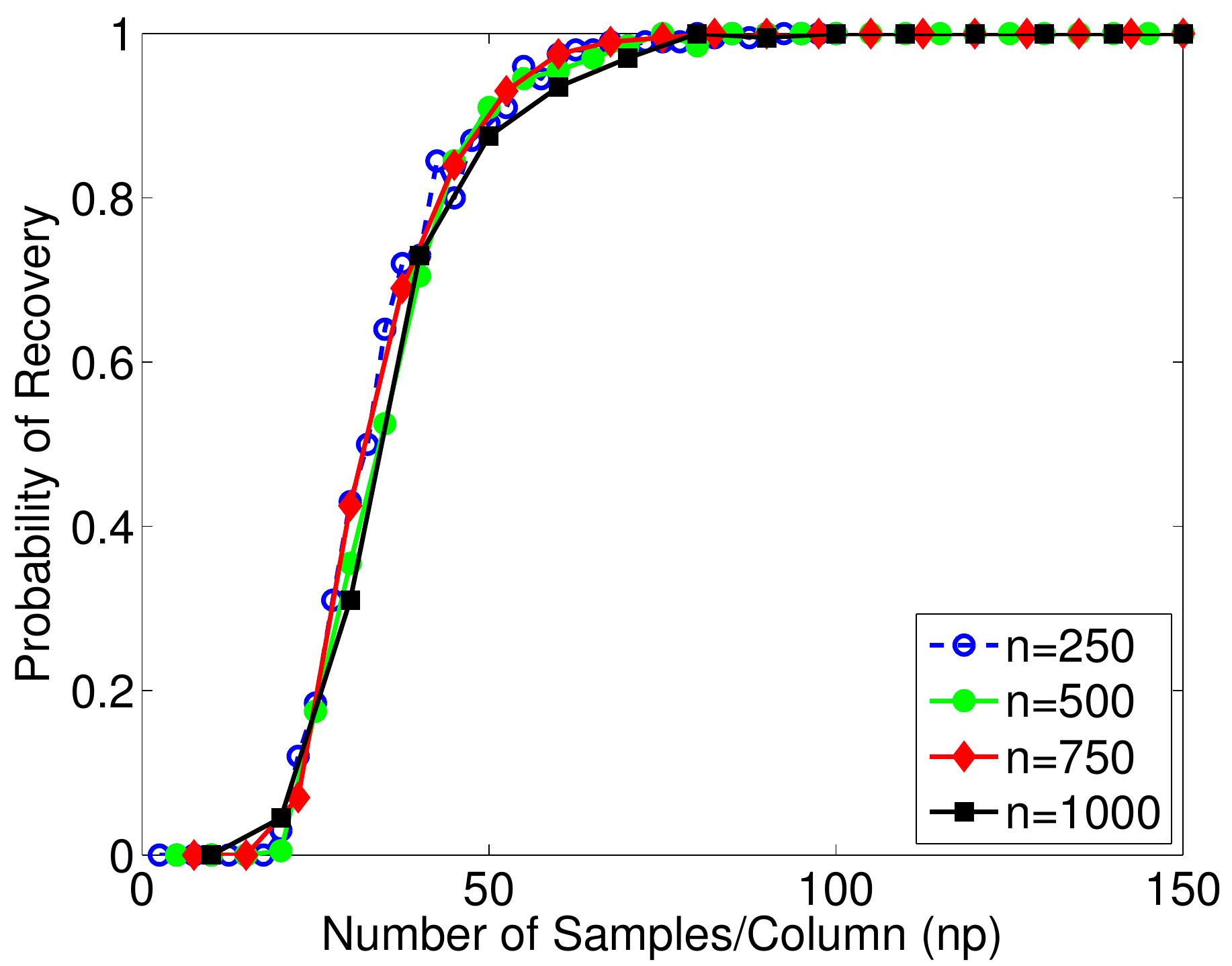}
\label{fig:exact_p_threshold_2}
}\subfigure[]{
\includegraphics[scale=0.19]{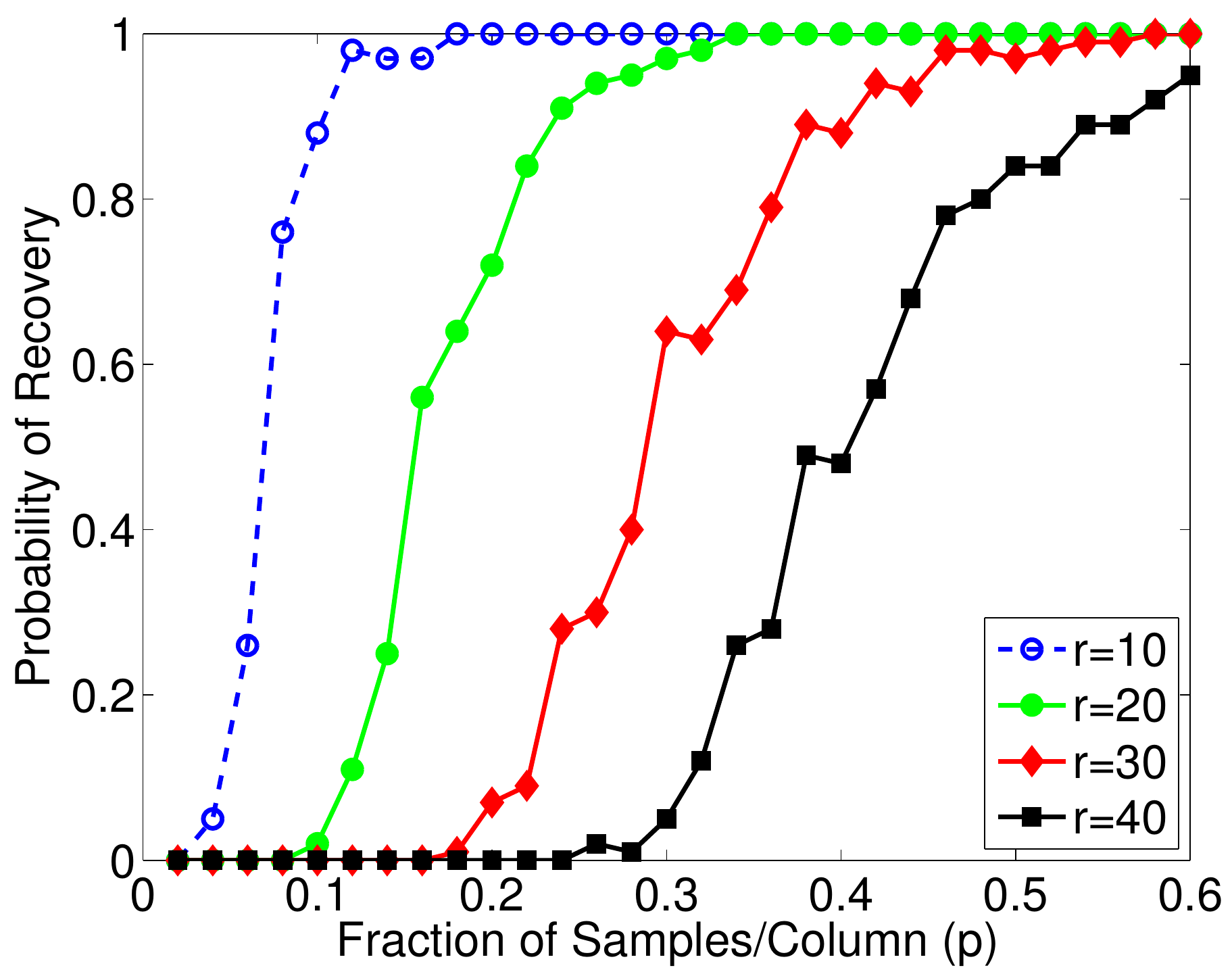}
\label{fig:exact_r_threshold_1}
}\subfigure[]{
\includegraphics[scale=0.19]{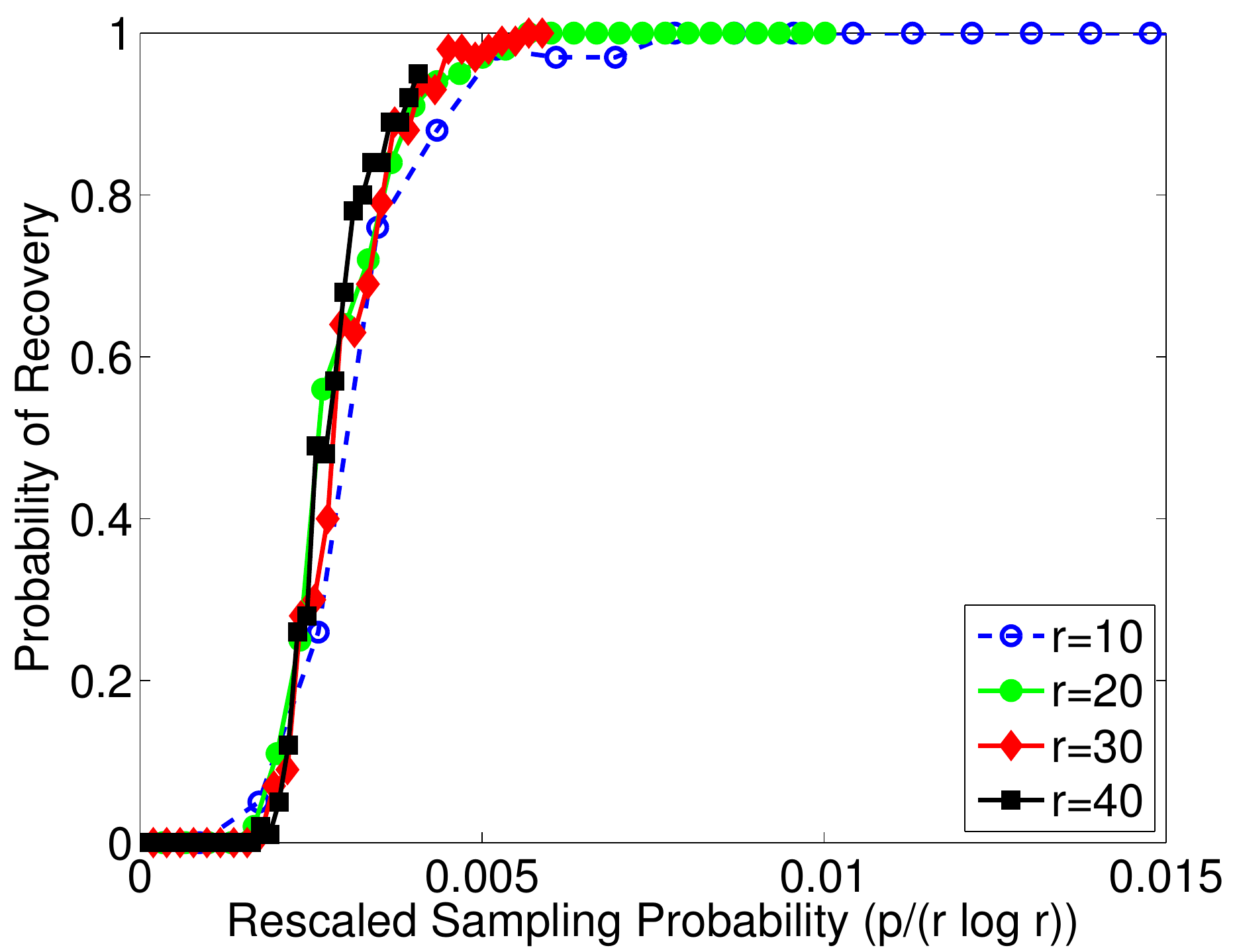}
\label{fig:exact_r_threshold_2}
}
\caption{\subref{fig:exact_p_threshold_1}: Probability of success of Algorithm~\ref{alg:exact_mc} versus fraction of samples per column ($p=m/d$) with $r=10, \mu_0=1$.
\subref{fig:exact_p_threshold_2}: Data from~\subref{fig:exact_p_threshold_1} plotted against samples per column, $m$.
\subref{fig:exact_r_threshold_1}: Probability of success of Algorithm~\ref{alg:exact_mc} versus fraction of samples per column ($p=m/d$) with $n=500, \mu_0=1$.
\subref{fig:exact_r_threshold_2}: Data from~\subref{fig:exact_r_threshold_1} plotted against rescaled sample probability $p/(r \log r)$.}
\label{fig:exact_simulations_1}
\end{figure}

We perform a number of simulations to analyze the empirical performance of both Algorithms~\ref{alg:exact_mc} and~\ref{alg:approx}.
The first set of simulations, in Figures~\ref{fig:exact_simulations_1} and~\ref{fig:exact_simulations_2}, examine the behavior of Algorithm~\ref{alg:exact_mc}.
We work with square matrices where the column space is spanned by binary vectors, constructed so that the matrix has the appropriate rank and coherence. 
The row space is spanned by either random gaussian vectors in the case of incoherent row space or a random collection of standard basis elements if we want high coherence. 

In the first two figures (\ref{fig:exact_p_threshold_1} and~\ref{fig:exact_p_threshold_2}) we study the algorithms dependence on the matrix dimension. 
For various matrix sizes, we record the probability of exact recovery as we vary the number of samples alloted to the algorithm.
We plot the probability of recovery as a function of the fraction of samples per column, denote by $p$, (Figure~\ref{fig:exact_p_threshold_1}) and as a function of the total samples per column $m$ (Figure~\ref{fig:exact_p_threshold_2}).
It is clear from the simulations that $p$ can decrease with matrix dimension while still ensuring exact recovery.
On the other hand, the curves in the second figure line up, demonstrating that the number of samples per column remains fixed for fixed probability of recovery.
This behavior is predicted by Theorem~\ref{thm:exact_ub}, which shows that the total number of samples scales linearly with dimension, so that the number of samples per column remains constant.

\begin{figure}
\subfigure[]{
\includegraphics[scale=0.19]{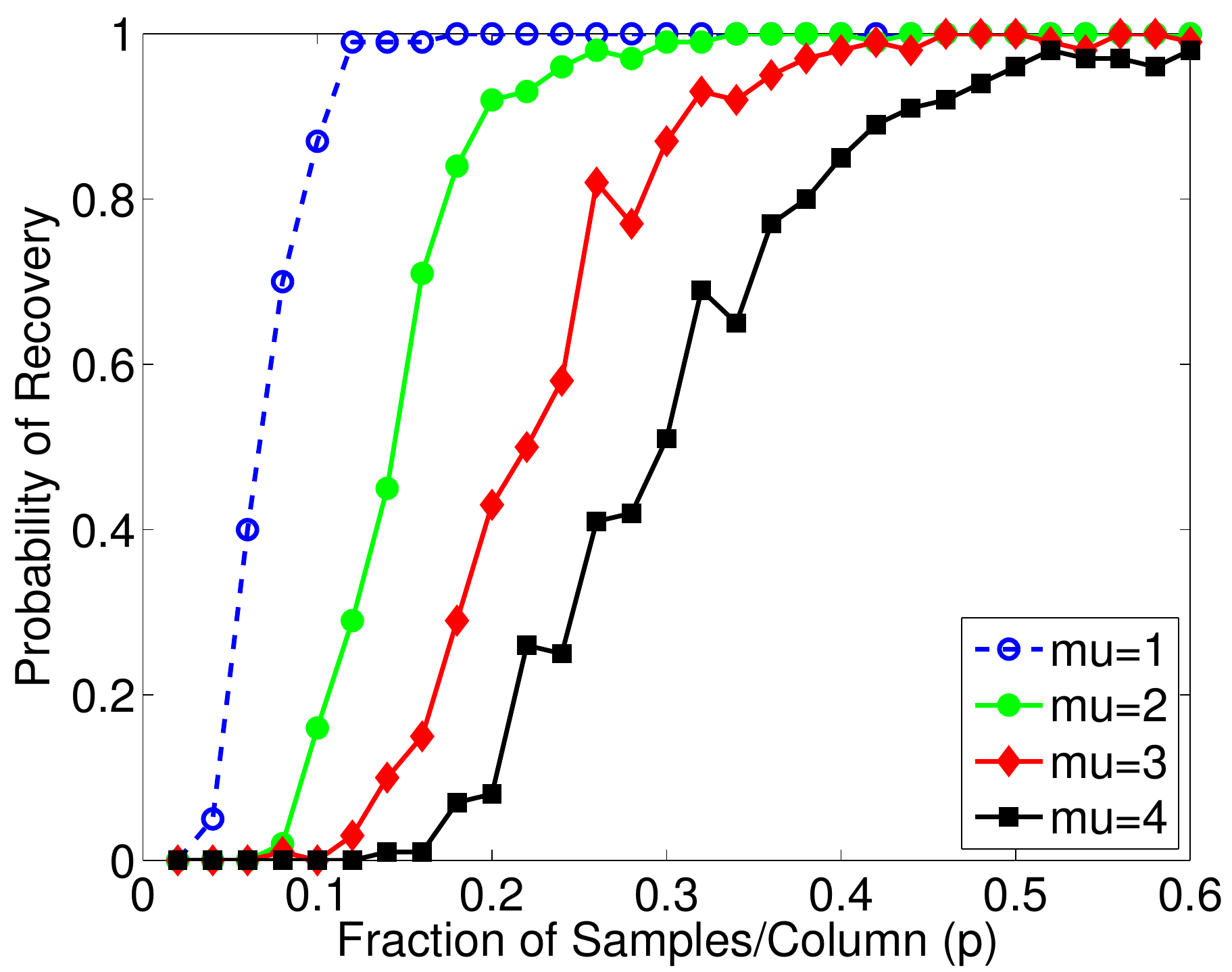}
\label{fig:exact_mu_threshold_1}
}\subfigure[]{
\includegraphics[scale=0.19]{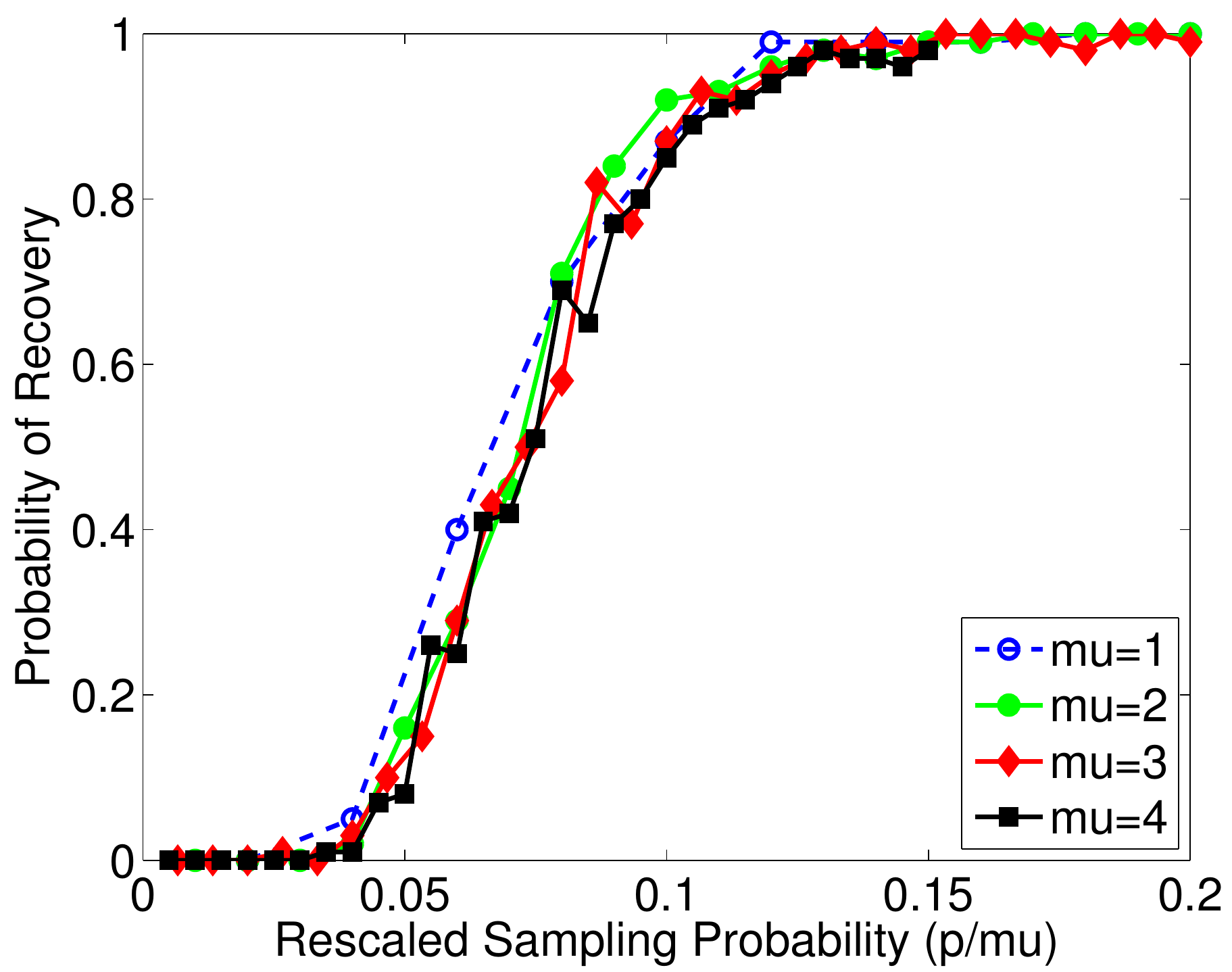}
\label{fig:exact_mu_threshold_2}
}\subfigure[]{
\includegraphics[scale=0.19]{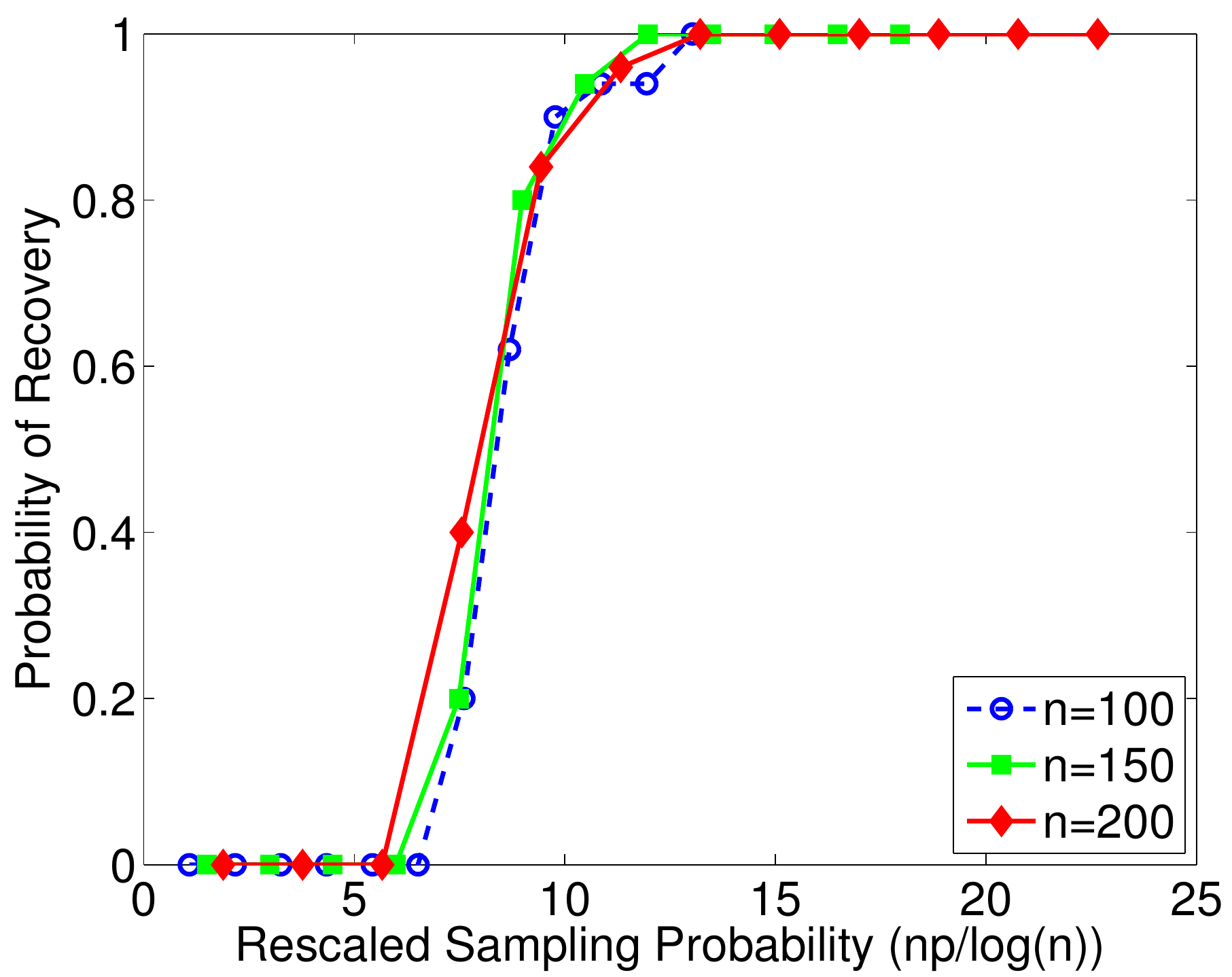}
\label{fig:exact_svt_threshold_2}
}\subfigure[]{
\includegraphics[scale=0.19]{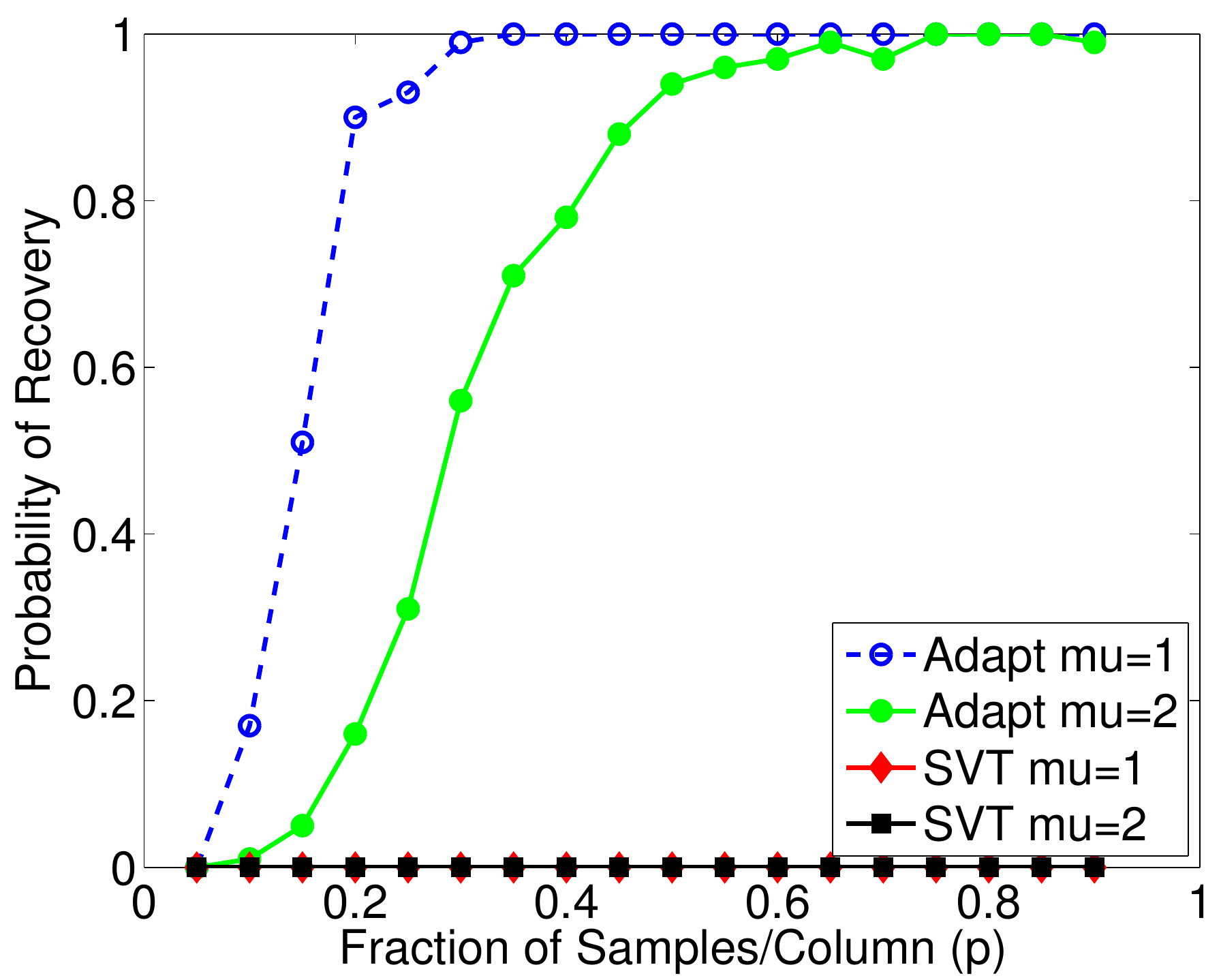}
\label{fig:exact_coherent_threshold}
}
\caption{\subref{fig:exact_mu_threshold_1}: Probability of success of Algorithm~\ref{alg:exact_mc} versus fraction of samples per column ($p=m/d$) with $n=500, r=10$.
\subref{fig:exact_mu_threshold_2}: Data from~\subref{fig:exact_mu_threshold_1} plotted against rescaled sampling probability $p/\mu_0$.
\subref{fig:exact_svt_threshold_2}: Probability of success of SVT versus rescaled sampling probability $np/\log(n)$ with $r=5, \mu_0=1$. 
\subref{fig:exact_coherent_threshold}: Probability of sucess of Algorithm~\ref{alg:exact_mc} and SVT versus sampling probability for matrices with highly coherent row space with $r=5, n=100$.}
\label{fig:exact_simulations_2}
\end{figure}

In Figures~\ref{fig:exact_r_threshold_1} and~\ref{fig:exact_r_threshold_2} we show the results of a similar simulation, instead varying the matrix rank $r$, with dimension fixed at $500$. 
The first figure shows that the fraction of samples per column must increase with rank to ensure successful recovery while second shows that the ratio $p/(r \log r)$ governs the probability of success.
Figures~\ref{fig:exact_mu_threshold_1} and~\ref{fig:exact_mu_threshold_2} similarly confirm a linear dependence between the incoherence parameter $\mu_0$ and the sample complexity.
Notice that the empirical dependence on rank is actually a better than what is predicted by Theorem~\ref{thm:exact_ub}, which suggests that $r \log^2r$ is the appropriate scaling.
Our theorem does seem to capture the correct dependence on the coherence parameter. 

In the last two plots we compare Algorithm~\ref{alg:exact_mc} against the Singular Value Thresholding algorithm (SVT) of \citet{cai2010singular}.
The SVT algorithm is a non-adaptive iterative algorithm for nuclear norm minimization from a set of uniform-at-random observations. 
In Figure~\ref{fig:exact_svt_threshold_2}, we show that the success probability is governed by $np/\log(n)$, 
which is predicted by the existing analysis of the nuclear norm minimization program.
This dependence is worse than for Algorithm~\ref{alg:exact_mc}, whose success probability is governed by $np$ as demonstrated in Figure~\ref{fig:exact_p_threshold_2}.
Finally, in Figure~\ref{fig:exact_coherent_threshold}, we record success probability versus sample complexity on matrices with maximally coherent row spaces. 
The simulation shows that our algorithm can tolerate coherent row spaces while the SVT algorithm cannot. 

\begin{figure}
\subfigure[]{
\includegraphics[scale=0.2]{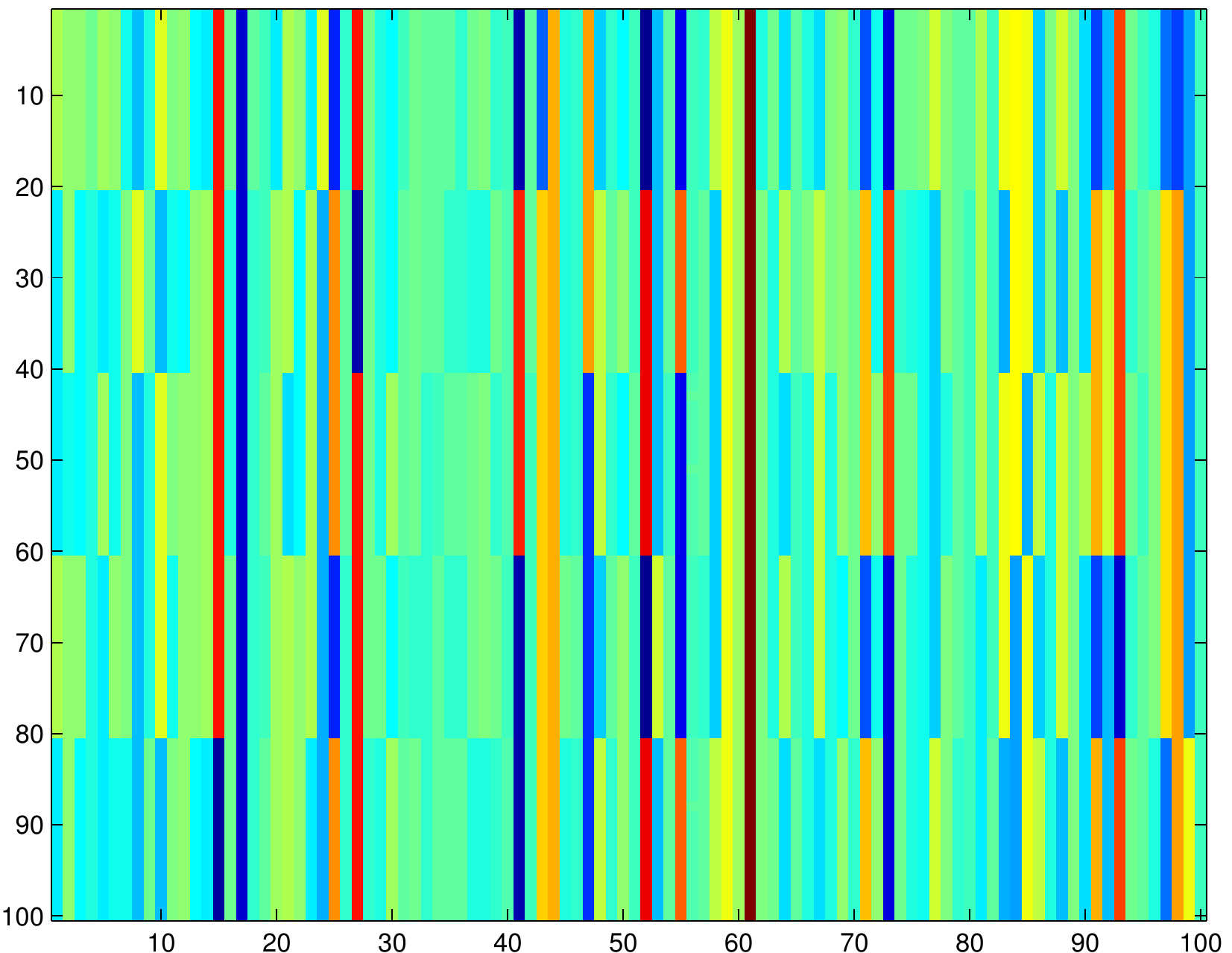}
\label{fig:incoherent_matrix}
} \subfigure[]{
\includegraphics[scale=0.2]{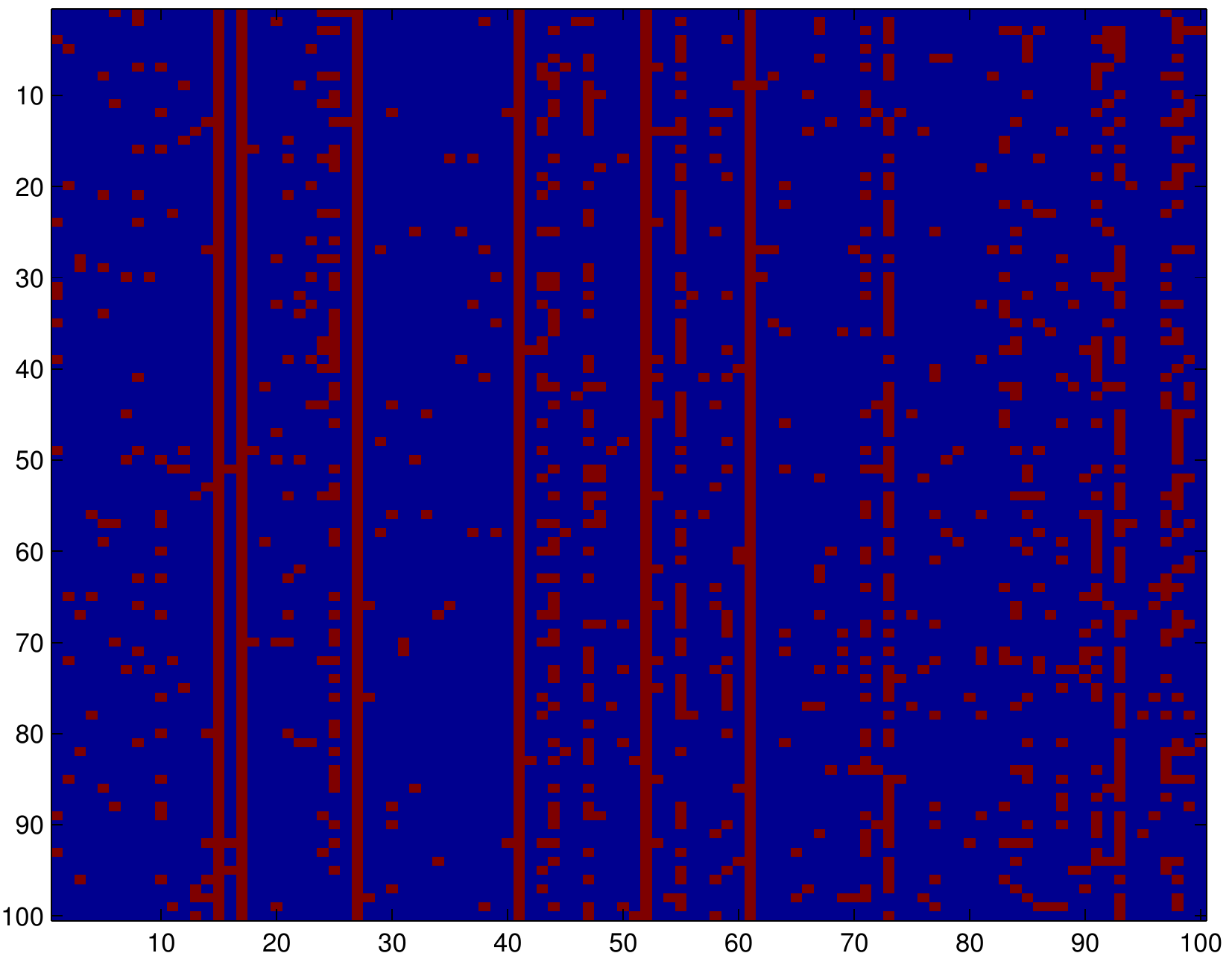}
\label{fig:sampling_pattern}
} \subfigure[]{
\includegraphics[scale=0.18]{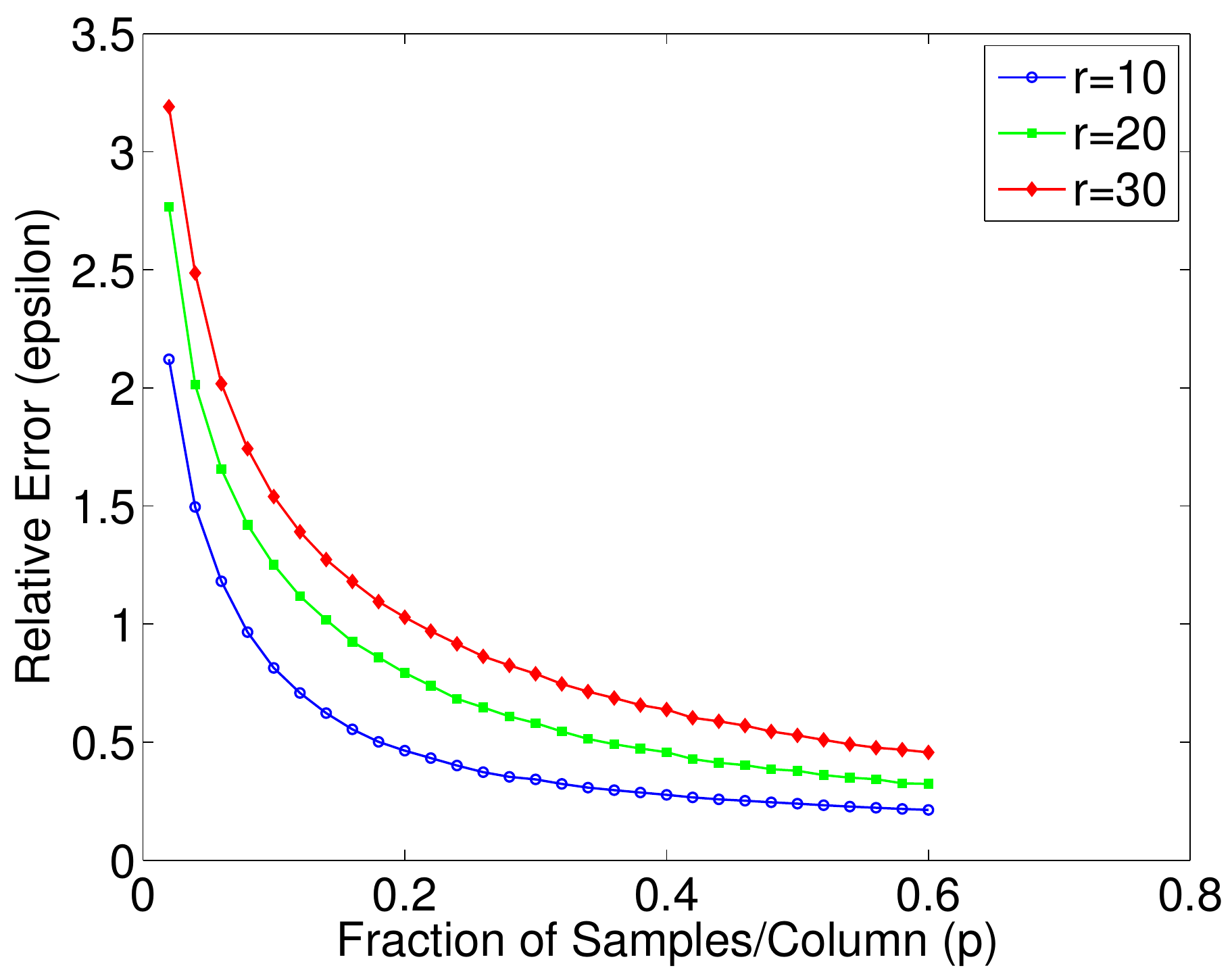}
\label{fig:approx_r_threshold_1}
} \subfigure[]{
\includegraphics[scale=0.18]{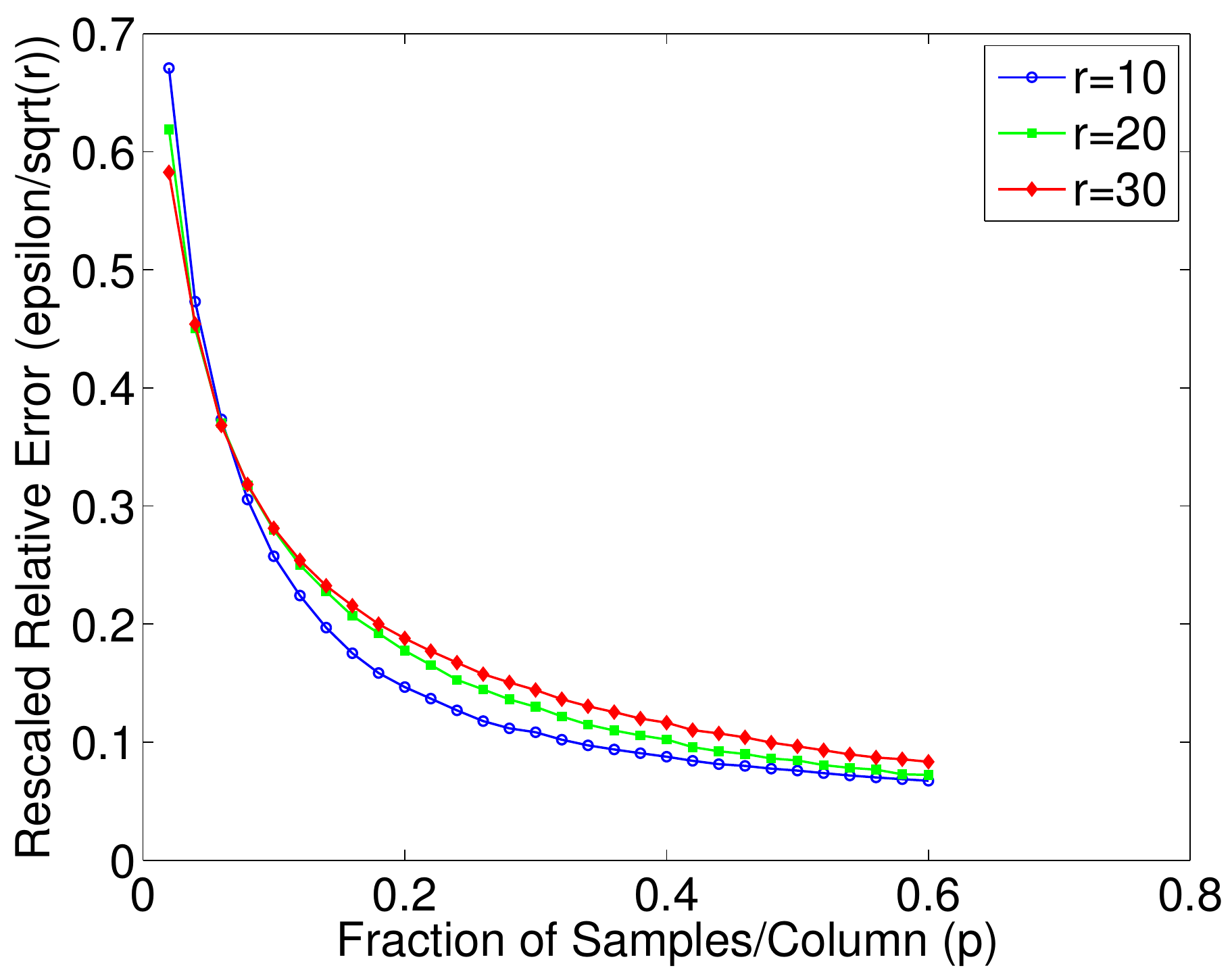}
\label{fig:approx_r_threshold_2}
}
\caption{\subref{fig:incoherent_matrix}: An example matrix with with highly non-uniform colum norms and \subref{fig:sampling_pattern} the sampling pattern of Algorithm~\ref{alg:approx}. 
\subref{fig:approx_r_threshold_1}: Relative error as a function of samping probability $p$ for different target rank $r$ ($\mu=1$).
\subref{fig:approx_r_threshold_2}: The same data where the $y$-axis is instead $\epsilon/\sqrt{r}$.}
\label{fig:approx_simulations_1}
\end{figure}

For Algorithm~\ref{alg:approx}, we display the results of a similar set of simulations in Figures~\ref{fig:approx_simulations_1} and~\ref{fig:approx_simulations_2}. 
Here, we construct low rank matrices whose column spaces are spanned by binary vectors and whose columns are also constant in magnitude on their support.
The length of the columns is distributed either log-normally, resulting in non-uniform column lengths, or uniformly between 0.9 and 1.1. 
We then corrupt this low rank matrix by adding a gaussian matrix whose entries have variance $\frac{1}{dn}$. 
In Figure~\ref{fig:incoherent_matrix} we show a matrix constructed via this process and in Figure~\ref{fig:sampling_pattern} we show the set of entries sampled by Algorithm~\ref{alg:approx} on this input. 
From the plots, it is clear that the algorithm focuses its measurements on the columns with high energy, while using very few samples to capture the columns with lower energy. 

In Figure~\ref{fig:approx_r_threshold_1}, we plot the relative error, which is the $\epsilon$ in Equation~\ref{eq:approx_risk}, as a function of the average fraction of samples per column (averaged over columns, as we are using non-uniform sampling) for $500 \times 500$ matrices of varying rank. 
In the next plot, Figure~\ref{fig:approx_r_threshold_2}, we rescale the relative error by $\sqrt{r}$, to capture the dependence on rank predicted by Theorem~\ref{thm:approx_adv}.
As we increase the number of observations, the relative error decreases quite rapidly.
Moreover, the algorithm needs more observations as the target rank $r$ increases. 
Qualitatively both of these effects are predicted by Theorem~\ref{thm:approx_adv}.
Lastly, the fact that the curves in Figure~\ref{fig:approx_r_threshold_2} nearly line up suggests that the relative error $\epsilon$ does scale with $\sqrt{r}$. 

\begin{figure}
\subfigure[]{
\includegraphics[scale=0.18]{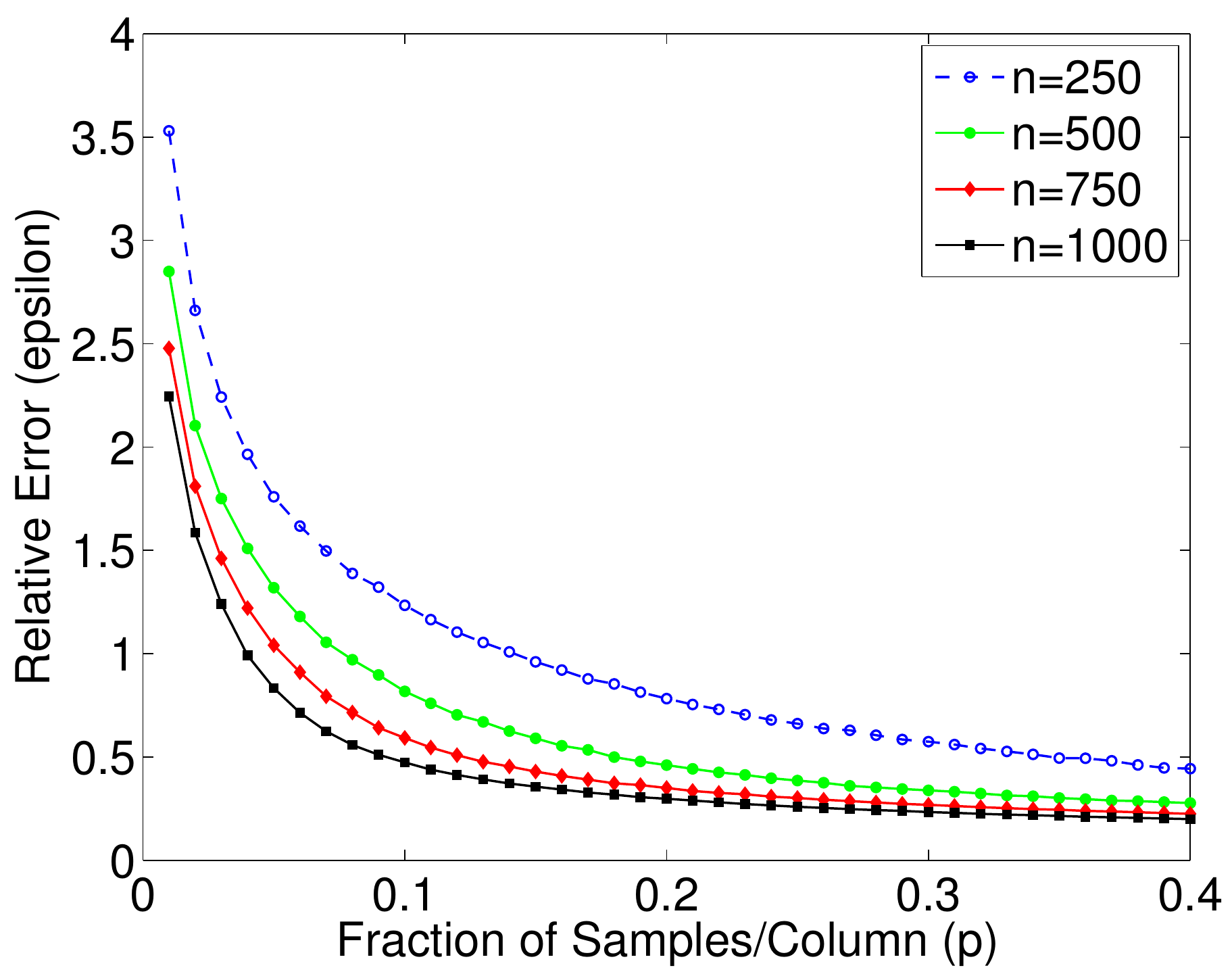}
\label{fig:approx_p_threshold_1}
} \subfigure[]{
\includegraphics[scale=0.18]{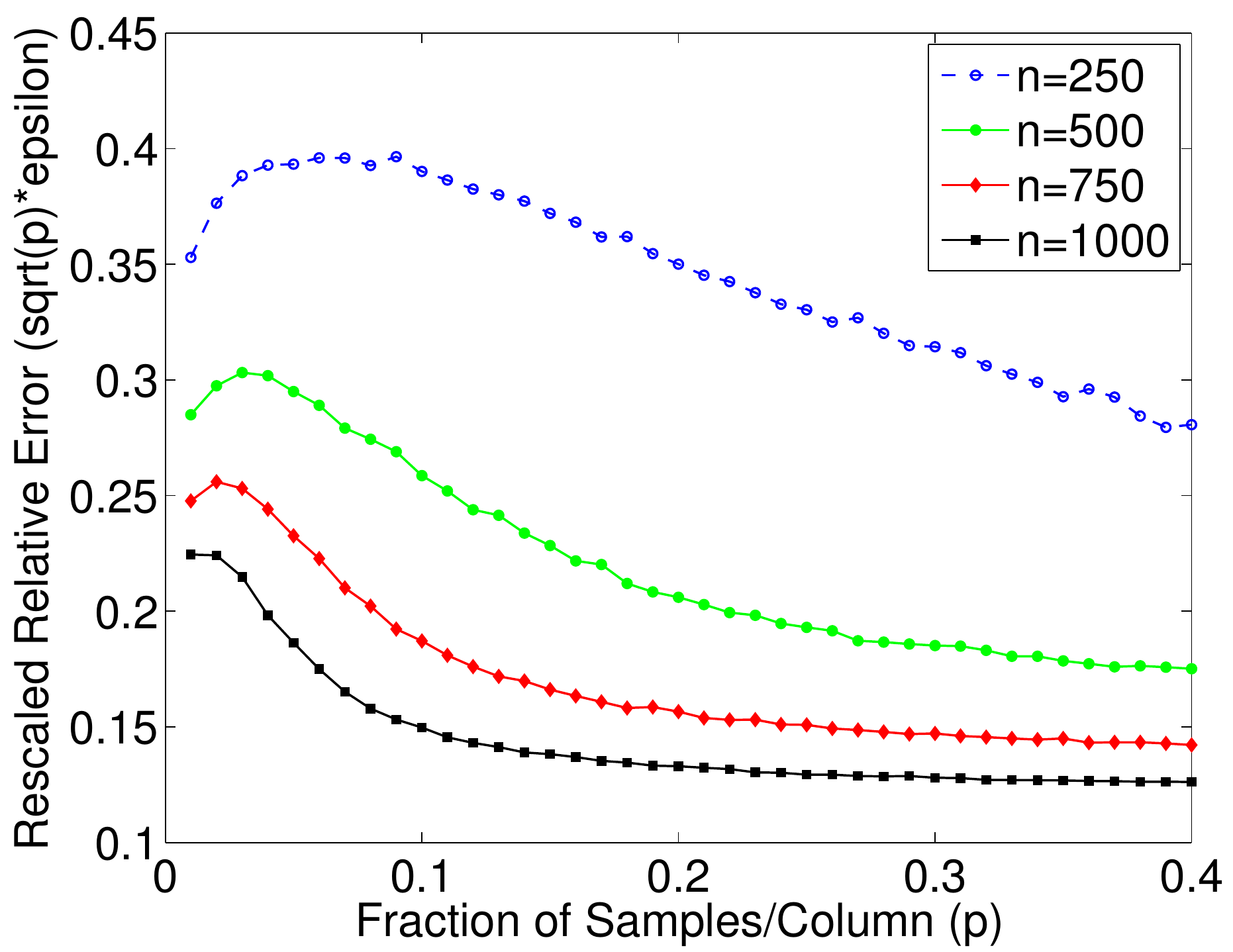}
\label{fig:approx_p_threshold_3}
} \subfigure[]{
\includegraphics[scale=0.18]{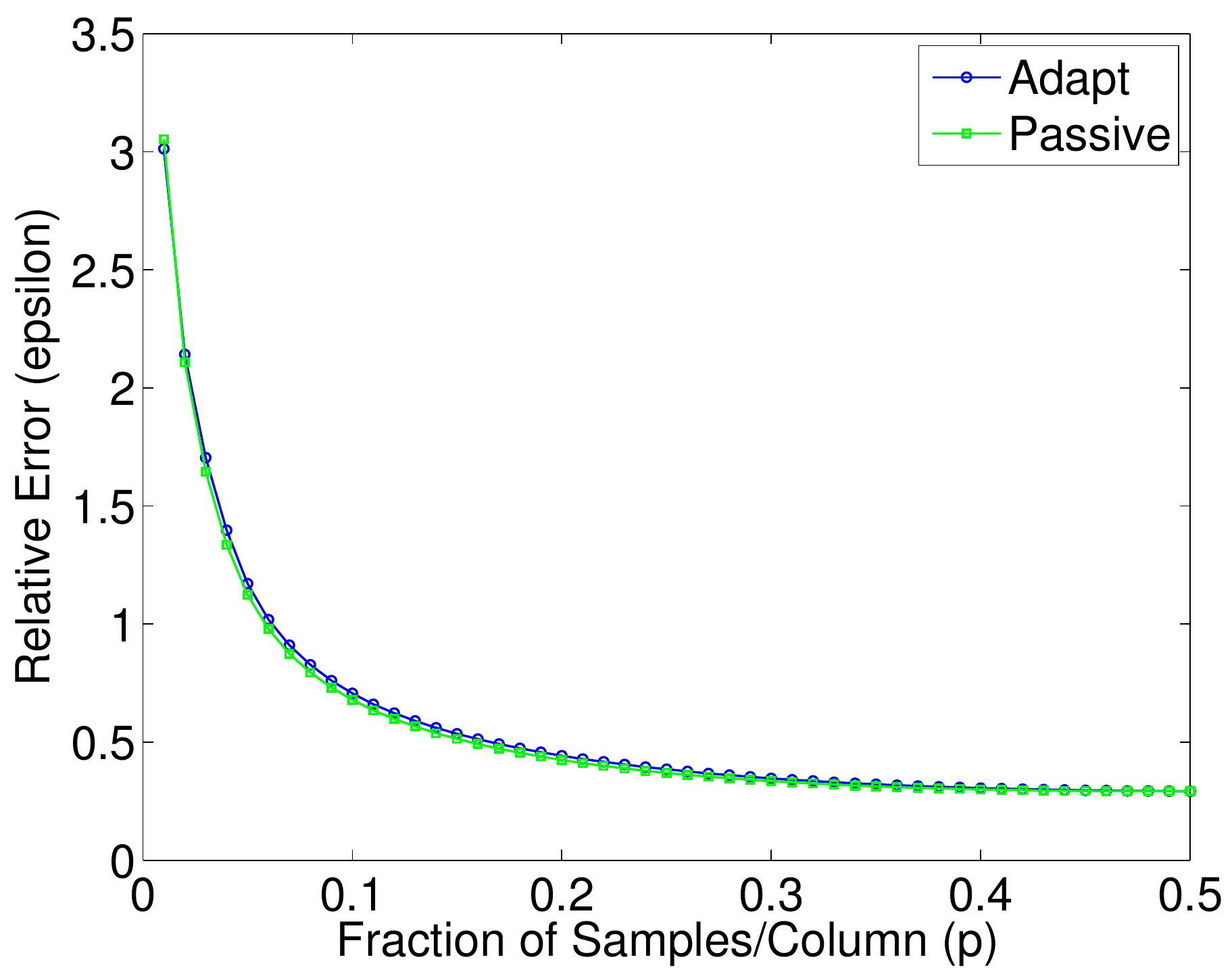}
\label{fig:approx_flat_comparison}
}\subfigure[]{
\includegraphics[scale=0.18]{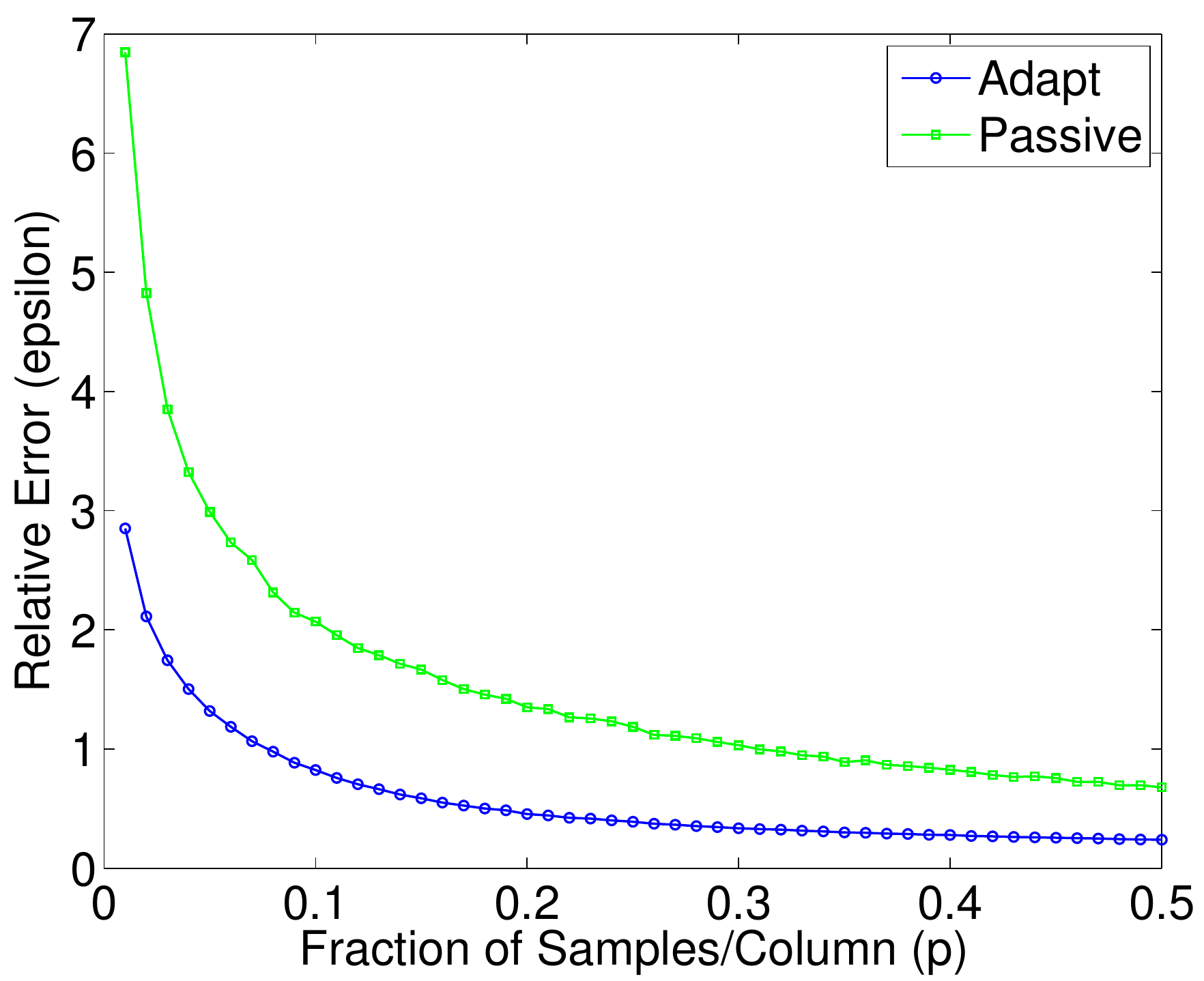}
\label{fig:approx_spiky_comparison}
}
\caption{\subref{fig:approx_p_threshold_1}: Relative error of Algorithm~\ref{alg:approx} as a function of sampling probability $p$ for different size matrices with fixed target rank $r=10$ and $\mu=1$. 
\subref{fig:approx_p_threshold_3}: The same data where the $y$-axis is instead $\sqrt{p}\epsilon$.
\subref{fig:approx_flat_comparison}: Relative error for adaptive and passive sampling on matrices with uniform column lengths (column coherence $\mu=1$ and column norms are uniform from $[0.9, 1.1]$).
\subref{fig:approx_flat_comparison}: Relative error for adaptive and passive sampling on matrices with highly nonuniform column lengths (column coherence $\mu=1$ and column norms are from a standard Log-Normal distribution). }
\label{fig:approx_simulations_2}
\end{figure}

In Figure~\ref{fig:approx_p_threshold_1}, we plot the relative error as a function of the average fraction of samples, $p$, per column for different matrix sizes. 
We rescale this data by plotting the $y$-axis in terms of $\sqrt{p}\epsilon$ (Figure~\ref{fig:approx_p_threshold_3}).
From the first plot, we see that the error quickly decays, while a smaller fraction of samples are needed for larger problems. 
In the second plot, we see that rescaling the error by $\sqrt{p}$ has the effect of flattening out all of the curves, which suggests that the relationship between $\epsilon$ and the number of samples is indeed $\epsilon\sqrt{p} \asymp 1$ or that $\epsilon \asymp \frac{1}{\sqrt{p}}$.
This phenomenon is predicted by Proposition~\ref{cor:approx_stoch}. 


In the last set of simulations, we compare our algorithm with an algorithm that first performs uniform sampling and then hard thresholds the singular values to build a rank $r$ approximation. 
In Figure~\ref{fig:approx_flat_comparison}, we use matrices with uniform column norms, and observe that both algorithms perform comparably.
However, in Figure~\ref{fig:approx_spiky_comparison}, when the column norms are highly non-uniform, we see that Algorithm~\ref{alg:approx} dramatically outperforms the passive sampling approach.
This confirms our claim that adaptive sampling leads to better approximation when the energy of the matrix is not uniformly distributed. 

\section{Proofs}
\label{sec:proofs}
In this section we provide the proofs of our main theorems. 
We defer some concentration results and some details to the appendix.

\subsection{Proof of Theorem~\ref{thm:exact_ub}}
The proof is identical to the proof by \citet{krishnamurthy2013low}, with an improved concentration of measure argument.
We reproduce some of the details here.

The main step in the proof analyzes of the test based on the projection $\|x_{t \Omega}- \Pcal_{U_\Omega}x_{t\Omega}\|_2^2$. 
Using various versions of Bernstein's inequality we are able to prove the following theorem, which builds off of \citet{balzano2010high} and \citet{krishnamurthy2013low}.
\begin{theorem}
Let $U$ be an $r$-dimensional subspace of $\RR^d$ and $y = x+v$ where $x \in U$ and $v \in U^\perp$. 
Fix $\delta > 0$ and $m \ge \max\{ \frac{8}{3} r \mu(U) \log(2d/\delta), 4 \mu(v)\log(1/\delta)\}$ and let $\Omega$ be an index set of $m$ entries sampled uniformly with replacement from $[d]$. 
With probability $\ge 1 - 4\delta$:
\begin{align}
\frac{m(1-\alpha) - r \mu(U) \frac{\beta}{1-\gamma}}{d} \|v\|_2^2 \le \|y_{\Omega} - \Pcal_{U_{\Omega}} y_{\Omega}\|_2^2 \le (1+\alpha)\frac{m}{d}\|v\|_2^2
\end{align}
where $\alpha = \sqrt{2 \frac{\mu(v)}{m}\log(1/\delta)} + \frac{2\mu(v)}{3m}\log(1/\delta), \beta = (1+2\log(1/\delta))^2$, and $\gamma = \sqrt{\frac{8 r \mu(U)}{3m} \log(2d/\delta)}$. 
\label{thm:subspace_proj}
\end{theorem}
This result showcases much stronger concentration of measure than the result of~\citet{balzano2010high}. 
The main difference is in the definitions of $\alpha$ and $\beta$, which in their work have worse dependence on the coherence parameter $\mu(v)$. 
Some of these improvements were established by \citet{krishnamurthy2013low}, but our result further improves the dependence on $\beta$, which will play out into our stronger sample complexity guarantee for the matrix completion algorithm.
In terms of proof, we use scalar, vector, and matrix Bernstein's inequality to control the terms in the decomposition:
\begin{align*}
\|y_{\Omega} - \Pcal_{U_{\Omega}} y_{\Omega}\|_2^2 \ge \|v_{\Omega}\|_2^2 - \|(U_{\Omega}^TU_{\Omega})^{-1}\|_2 \|U_{\Omega}^Tv_{\Omega}\|_2^2.
\end{align*}
The decomposition is valid provided that $U_{\Omega}^TU_{\Omega}$ is invertible, which we will account for. 

The above result, followed by some algebraic manipulations, yields the following corollary, which we use in the analysis of the Algorithm~\ref{alg:exact_mc}:
\begin{corollary}
Suppose that $\tilde{U}$ is a subspace of $U$ and $x_t \in U$ but $x_t \notin \tilde{U}$.
Observe a set of coordinates $\Omega \subset [d]$ of $m$ entries sampled uniformly at random with replacement.
If $m \ge 32 r \mu_0 \log^2(2r/\delta)$ then with probability $\ge 1-4\delta$, $\|x_{t\Omega} - \Pcal_{\tilde{U}_{\Omega}}x_{t\Omega}\|_2 > 0$.
If $x_t \in \tilde{U}$, then conditioned on the fact that $U_{\Omega}^TU_{\Omega}$ is invertible, $\|x_{t\Omega} - \Pcal_{\tilde{U}_{\Omega}}x_{t\Omega}\|_2 = 0$ with probability 1.
\label{cor:exact_cor}
\end{corollary}
\begin{proof}
The second statement follows from the fact that if $x_t \in \tilde{U}$, then $x_{t\Omega} \in \tilde{U}_{\Omega}$, so the projection onto the orthogonal complement is identically zero.
As for the first statement, we apply Theorem~\ref{thm:subspace_proj}, noting that the conditions on $m$ are satisfied.

We now verify that the lower bound is strictly positive.
We will use the fact that any vector $v$ in $U$ has coherence $\mu(v) \le r \mu_0$ and similarly any subspace $\tilde{U} \subset U$ has $\textrm{dim}(\tilde{U}) \mu(\tilde{U}) \le r \mu_0$. 
Plugging in $m$ into the definition $\alpha, \gamma$, and using the previous facts, we see that $\alpha < 1/2$ and $\gamma < 1/3$.
We are left with:
\begin{align*}
\|x_{t\Omega}- \Pcal_{\tilde{U}_{\Omega}}x_{t\Omega}\|_2^2 \ge \frac{1}{d}\left(\frac{m}{2} - \frac{3r \mu \beta}{2}\right)
\end{align*}
and the lower bound is strictly positive whenever $3r \mu \beta \le m$. 
Plugging in the definition of $\beta$, we see that this relation is also satisfied, concluding the proof. 
\end{proof}

We are now ready to prove Theorem~\ref{thm:exact_ub}.
First notice that our estimate $U$ for the column space is always a subspace of the true column space, since we only ever add in fully observed vectors that live in the column space.
Also notice that we only resample the set $\Omega$ at most $r+1$ times, since the matrix is exactly rank $r$, and we only resample when we find a linearly independent column. 
Thus with probability $1-(r+1)\delta$, by application of Lemma~\ref{lem:gamma} from the appendix, all of the matrices $\tilde{U}_{\Omega}^T\tilde{U}_{\Omega}$ are invertible. 

When processing the $t$th column, one of two things can happen.
Either $x_t$ lives in our current estimate for the column space, in which case we know from the above corollary that with probability $1$, $\|x_{t \Omega} - \Pcal_{U_\Omega}x_{t\Omega}\|^2 = 0$.
This holds since we have already conditioned on the fact that $U_{\Omega}^TU_\Omega$ is invertible. 
When this happens we do not obtain additional samples and just need to ensure that we reconstruct $x_t$, which we will see below.
If $x_t$ does not live in $U$, then with probability $\ge 1 - 4\delta$ the estimated projection is strictly positive, in which case we fully observe the new direction $x_t$ and augment our subspace estimate.
In fact, this failure probability includes the event that $U_{\Omega}^TU_\Omega$ is not invertible. 

Since $X$ has rank at most $r$, this latter case can happen no more than $r$ times, and via a union bound, the failure probability is $\le 4r\delta + \delta$. 
Here the last factor of $\delta$ ensures that the last subsampled projection operator is well behaved.
In other words, with probability $\ge 1 - 4r\delta-\delta$, our estimate $U$ at the end of the algorithm is exactly the column space of $X$. 

The vectors that were not fully observed are recovered exactly as long as $(U_{\Omega}^TU_{\Omega})^{-1}$ is invertible.
This follows from the fact that, if $x_t \in U$, we can write $x_t = U\alpha_t$ and we have:
\[
\hat{x}_t = U (U_{\Omega}^TU_{\Omega})^{-1} U_{\Omega}^T U_{\Omega}\alpha_t = U\alpha_t = x_t
\]
We already accounted for the probability that these matrices are invertible. 
We showed above that the total failure probability is $\le 5r\delta$ and solving for $\delta$ in the sample complexity in Corollary~\ref{cor:exact_cor}, we find that:
\begin{align*}
\delta \le 2r \exp\left\{ - \sqrt{\frac{m}{32r \mu_0}}\right\},
\end{align*}
which gives the risk bound.

For the running time, per column, the dominating computational costs involve the projection $\mathcal{P}_{\tilde{U}_\Omega}$ and the reconstruction procedure. 
The projection involves several matrix multiplications and the inversion of a $r \times r$ matrix, which need not be recomputed on every iteration. 
Ignoring the matrix inversion, this procedure takes at most $O(m r)$ time per column, since the vector and the projector are subsampled to $m$-dimensions, for a total running time of $O(nmr)$. 
At most $r$ times, we must recompute $(U_{\Omega}^TU_{\Omega})^{-1}$, which takes $O(r^2m)$, contributing a factor of $O(r^3m)$ to the total running time. 
Finally, we run the Gram-Schmidt process once over the course of the algorithm, which takes $O(dr^2)$ time.

\subsection{Proof of Theorem~\ref{thm:exact_lb}}
The proof of the necessary condition in Theorem~\ref{thm:exact_lb} is based on a standard reduction-to-testing style argument. 
The high-level architecture is to consider a subset $\Xcal' \subset \Xcal$ of inputs and lower bound the Bayes risk.
Specifically, if we fix a prior $\pi$ supported on $\Xcal'$,
\begin{align*}
R^\star &= \inf_{f \in \Fcal} \inf_{q \in \Qcal} \max_{X \in \Xcal} \PP_{\Omega \sim q}[f(\Omega, X_{\Omega}) \ne X]\\
& \ge \inf_{f \in \Fcal} \inf_{q \in \Qcal} \EE_{\Omega \sim q, X \sim \pi}[\PP_{f}[f(\Omega, X_{\Omega}) \ne X]]\\
& \ge \inf_{f \in \Fcal} \min_{\Omega: |\Omega| = m} \EE_{X \sim \pi}[\PP_f[f(\Omega, X_{\Omega}) \ne X]]
\end{align*}
The first step is a standard one in information theoretic lower bounds and follows from the fact that the maximum dominates any expectation over the same set. 
The second step is referred to as Yao's Minimax Principle in the analysis of randomized algorithms, which says that one need only consider deterministic algorithms if the input is randomized. 
It is easily verified by the fact that in the second line, the inner expression is linear in $q$, so it is minimized on the boundary of the simplex, which is a deterministic choice of $\Omega$. 
We use $\PP_f$ to emphasize that $f$ can be randomized, although it will suffice to consider deterministic $f$.

Let $\pi$ be the uniform distribution over $\Xcal' \subset \Xcal$.
The minimax risk is lower bounded by:
\[
R^\star \ge 
1 - \max_{\Omega} \EE_{X \sim \pi}|\{X' \in \Xcal' | X'_\Omega  = X_\Omega\}|^{-1}
\]
since if there is more than one matrix in $\Xcal'$ that agrees with $X$ on $\Omega$, the best any estimator could do is guess. 
Notice that since $X$ is drawn uniformly, this is equivalent to considering an $f$ that deterministically picks on matrix $X' \in \Xcal'$ that agrees with the observations. 

To upper bound the second term, define $\Ucal_{\Omega} = \{X \in \Xcal' : |\{X' \in \Xcal'| X'_{\Omega} = X_{\Omega}\}| = 1\}$ which is the set of matrices that are uniquely identified by the entries $\Omega$.
Also set $\Ncal_{\Omega} = \Xcal' \setminus \Ucal_\Omega$, which is the set of matrices that are not uniquely identified by $\Omega$. 
We may write:
\begin{align*}
\max_{\Omega} \EE_{X \sim \pi}|\{X' \in \Xcal' | X'_\Omega  = X_\Omega\}|^{-1} \le 
\max_{\Omega} \frac{1}{2} + \frac{|\Ucal_\Omega|}{2|\Xcal'|}
\end{align*}
Since if $X \in \Ncal_\Omega$, there are at least two matrices that agree on those observations, so the best estimator is correct with probability no more than $1/2$. 

We now turn to constructing a set $\Xcal'$. 
Set $l = \frac{d}{r \mu_0}$.
The left singular vectors $u_1, \ldots, u_{r-1}$ will be constant on $\{1, \ldots, l\}, \{l+1, \ldots, 2l\}$ etc. while the first $r-1$ right singular vectors $v_1, \ldots, v_{r-1}$ will be the first $r-1$ standard basis elements. 
We are left with:
\begin{align*}
d - (r-1)l = d - \frac{r-1}{r}\frac{d}{\mu_0} \triangleq d c_1,
\end{align*}
coordinates where we will attempt to hide the last left singular vector.
Here we defined $c_1 = 1 - \frac{r-1}{r\mu_0}$, which is not a constant, but will ease the presentation. 
For $u_r$, we pick $l$ coordinates out of the $dc_1$ remaining, pick a sign for each and let $u_r$ have constant magnitude on those coordinates. 
There are $2^l {dc_1 \choose l}$ possible choices for this vector. 
The last right singular vector is one of the $n-r$ remaining standard basis vectors.
Notice that our choice of $l$ ensures that every matrix in this family meets the column space incoherence condition.

To upper bound $|\Ucal_{\Omega}|$ notice that since $u_r$ can have both positive and negative signs, a matrix is uniquely identified only if all of the entries corresponding to the last singular vector are observed.
Thus observations in the $t$th column only help to identify matrices whose last rank was hidden in that column.
If we use $m_t$ observations on the $t$th column, we uniquely identify $2^l{m_t \choose l}$ matrices, where ${m_t \choose l} = 0$ if $m_t < l$. 
In total we have:
\begin{align*}
|\Xcal'| = (n-r)2^l{dc_1 \choose l} \qquad \textrm{and} \qquad |\Ucal_{\Omega}| = 2^l \sum_{i=r}^{n}{m_i \choose l} 
\end{align*}

We are free to choose $m_i$ to maximize $|\Ucal_\Omega|$ subject to the constraints $m_i \le dc_1$ and $\sum_i m_i \le m$, the total sensing budget. 
Optimizing over $m_i$ is a convex maximization problem with linear constraints, and consequently the solution is on the boundary.
By symmetry, this means that that best sampling pattern is to observe columns in their entirety and devote the remaining observations to one more column. 
With $m$ observations, we can observe $\frac{m}{c_1 n}$ columns fully, leading to the bounds:
\begin{align*}
|\Ucal_\Omega| \le 2^l\lceil \frac{m}{c_1n}\rceil {n c_1 \choose l}, \qquad \textrm{and} \qquad
\frac{|\Ucal_\Omega|}{|\Xcal'|} \le \lceil \frac{m}{c_1 n}\rceil \frac{1}{n_2 -r},
\end{align*}
which, after plugging in for $c_1$, leads to the lower bound on the risk. 

\subsection{Proof of Theorem~\ref{thm:approx_adv}}
To prove the main approximation theorem, we must analyze the three phases of the algorithm.
The analysis of the first phase is fairly straightforward: we show that under the incoherence assumption, one can compute a reliable estimate of each column norm from a very small number of measurements per column. 
For the second phase, we show that by sampling according to the re-weighted distribution using the column-norm estimates, the matrix $\tilde{X}$ is close to $X$ in spectral norm.
We then translate this spectral norm guarantee into a approximation guarantee for $\hat{X} = \tilde{X}_r$. 

Let us start with this translation. 
We use a lemma of \cite{achlioptas2007fast}.
\begin{lemma}[\cite{achlioptas2007fast}]
Let $A$ and $N$ be any matrices and write $\hat{A} = A + N$.
Then:
\begin{align*}
\|A - \hat{A}_k\|_2 &\le \|A - A_k\|_2 + 2 \|N_k\|_2\\
\|A - \hat{A}_k\|_F &\le \|A - A_k\|_F + \|N_k\|_F + 2 \sqrt{\|N_k\|_F\|A_k\|_F}
\end{align*}
\end{lemma}
The lemma states that if $\hat{A} - A$ is small, then the top $k$ ranks of $\hat{A}$ is nearly as good an approximation to $A$ as is the top $k$ ranks of $A$ itself. 
Notice that all of the error terms only depend on rank-$k$ matrices. 
We will use this lemma with $\tilde{X}$ and $X$ and of course with the target rank as $r$. 
We will soon show that $\|X - \tilde{X}\|_2 \le \epsilon \|X\|_F$, which implies:
\begin{align}
\|X - \hat{X}\|_F &\le \|X - X_r\| + \|(X - \tilde{X})_r\|_F + 2 \sqrt{\|(X - \tilde{X})_r\|_F \|X_r\|_F} \nonumber\\
& \le \|X - X_r\| + \sqrt{r}\|X - \tilde{X}\|_2 + 2 \sqrt{\sqrt{r}\|X - \tilde{X}\|_2 \|X\|_F} \nonumber\\
& \le \|X - X_r\| + \|X\|_F\left( \sqrt{r} \epsilon + 2 r^{1/4}\epsilon^{1/2}\right) \label{eq:approx_err_bd}
\end{align}
So if we can obtain a bound on $\|X - \tilde{X}\|_2$ of that form, we will have proved the theorem. 

As for Propositions~\ref{cor:approx_exact_cor} and~\ref{cor:approx_stoch}, the translation uses the first inequality of \citet{achlioptas2007fast}. 
If $X$ is rank $r$, the matrix $\hat{X} - X$ has rank at most $2r$, which means that:
\begin{align*}
\|\hat{X} - X\|_F \le \sqrt{2r}\|\hat{X} - X\|_2 \le 2\sqrt{2r}\|\tilde{X} - X\|_2 \le 2\sqrt{2r}\epsilon \|X\|_F
\end{align*}
For the second proposition, we first bound $\|\hat{X} - M\|_2$ and then use the same argument.
\begin{align*}
\|\hat{X} - M\|_2 &\le \|\hat{X} - X\|_2 + \|R\|_2 \le \|X - X_r\|_2 + 2\epsilon \|X\|_F + \|R\|_2\\
& \le 2\|R\|_2 + 2 \epsilon (\|M\|_F + \|R_{\Omega}\|_F).
\end{align*}
To arrive at the second line, we use the fact that $X_r$ is the best rank $r$ approximation to $X$, so $\|X - X_r\|_2 \le \|X - M\|_2 = \|R\|_2$. 
We also use the triangle inequality on the term $\|X\|_F$, but use the fact that since the algorithm never looked at $X$ on $\Omega^C$ it is fair to set $R_{\Omega^C} = 0$. 

Let us now turn to the first phase. 
In our analysis of the Algorithm~\ref{alg:exact_mc}, we proved that the norm of an incoherent vector can be approximated by subsampling. 
Specifically, Lemma~\ref{lem:alpha} shows that with high probability, the estimates $\hat{c}_t$ once appropriately rescaled are trapped between $(1-\alpha)c_t$ and $(1+\alpha)c_t$ where $\alpha= \sqrt{ 2\mu/m_1 \log(n/\delta)} + \frac{2\mu}{3m_1} \log(n/\delta)$. 
The same is of course true for $\hat{f}$.
Setting $m_1 \ge 32 \mu \log(n/\delta)$ we find that $\alpha \le 1/2$, meaning that by using in total $32 n \mu\log(n/\delta)$ samples in the first phase, we approximate the target sampling distribution to within a multiplicative factor of $1/2$ with probability $\ge 1- \delta$.

For the second pass, we must show that $\tilde{X}$ is close to $X$ in spectral norm.
Some calculations, that we defer to the appendix, give the following lemma:
\begin{lemma}
\label{lem:mb_app}
Provided that $(1-\alpha)c_t \le \frac{d}{m_1} \hat{c}_t \le (1+\alpha)c_t$ and $(1-\alpha)f \le \frac{d}{m_1}\hat{f} \le (1+\alpha)f$, with probability $\ge 1-\delta$:
\begin{align*}
\|\tilde{X} - X\|_2 \le \|X\|_F \sqrt{\frac{1+\alpha}{1-\alpha}}\left(\sqrt{\frac{4}{m_2}\max\left(\frac{d}{n}, \mu\right)\log\left(\frac{d+n}{\delta}\right)} + \frac{4}{3}\sqrt{\frac{d\mu}{m_2 n}}\log\left(\frac{d+n}{\delta}\right)\right)
\end{align*}
\end{lemma}

The adaptive sampling procedure has a dramatic effect on the bound in Lemma~\ref{lem:mb_app}.
If one sampled uniformly across the columns, then both terms grows with the squared norm of the largest column rather than with the average squared norms, which is much weaker when the energy of the matrix is concentrated on a few columns.
This is precisely when the row space is coherent.

To wrap up, recall that $1 \le \mu \le d$ and $n \ge d$. 
Setting $m_1 \ge 32 \mu \log(n/\delta)$ so that $\alpha \le 1/2$, the bound in Lemma~\ref{lem:mb_app} is dominated by:
\[
\|\tilde{X} - X\|_2 \le \|X\|_F\frac{10}{\sqrt{3}} \sqrt{\frac{\mu}{m_2}}\log\left(\frac{d+n}{\delta}\right).
\]
Returning to Equation~\ref{eq:approx_err_bd} we can now substitute in for $\epsilon$ and conclude the proof. 

\section{Discussion}
\label{sec:discussion}
This paper considers the two related problems of low rank matrix completion and matrix approximation. 
In both problems, we show how to use adaptive sampling to overcome uniformity assumptions that have pervaded the literature.
Our algorithms focus measurements on interesting columns (in the former, the columns that contain new directions and in the latter, the high energy columns) and have performance guarantees that are significantly better than any known passive algorithms in the absence of uniformity.
Moreover, they are competitive with state-of-the-art passive algorithms in the presence of uniformity.
Our algorithms are conceptually simple, easy to implement, and fairly scalable.

There are several interesting directions for future work and we mention two here. 
First, while we did discuss a lower bound on adaptive algorithms for matrix completion, we do not have a lower bound on the performance of adaptive algorithms for the matrix approximation problem.
Such a bound would gives us a better understanding on the fundamental limits of the matrix approximation problem.
More broadly, we are only beginning to understand the power of adaptive sampling and active learning in unsupervised settings and it would be interesting, both theoretically and practically, to develop this line of work further.


\section*{Acknowledgements}
This research is supported in part by NSF under grants IIS-1116458 and CAREER award IIS-1252412. 
AK is supported in part by an NSF Graduate Research Fellowship.

\bibliographystyle{plainnat}
\bibliography{mc}

\appendix
\section{Proof of Theorem~\ref{thm:subspace_proj}}
For completeness we provide the entire proof of Theorem~\ref{thm:subspace_proj} although apart from the improved concentration bounds, the proof is similar to that of Balzano \emph{et al.}~\cite{balzano2010high}.

We begin with the decomposition:
\begin{align}
\|y_\Omega - \Pcal_{U_\Omega}y_{\Omega}\|_2^2 = \|v_{\Omega}\|_2^2 - v_{\Omega}^TU_{\Omega}(U_\Omega^TU_\Omega)^{-1}U_\Omega^Tv_\Omega.
\end{align}
Next, let $W_\Omega^TW_\Omega = (U_\Omega^TU_\Omega)^{-1}$, which is valid provided that $U_\Omega^TU_\Omega$ is invertible (which we will subsequently ensure).
We have:
\begin{align*}
v_\Omega^TU_\Omega(U_\Omega^TU_\Omega)^{-1}U_\Omega^Tv_\Omega = \|W_\Omega U_\Omega^Tv_\Omega\|_2^2 \le \|W_\Omega|_2^2\|U_\Omega^Tv_\Omega\|_2^2 = \|(U_\Omega^TU_\Omega)^{-1}\| \| U_\Omega^Tv_\Omega\|_2^2,
\end{align*}
which means that:
\begin{align}
\|v_{\Omega}\|_2^2 - \|(U_\Omega^TU_\Omega)^{-1}\|\|U_\Omega^Tv_\Omega\|^2 \le \|y_\Omega - \Pcal_{U_\Omega}y_{\Omega}\|_2^2 \le \|v_{\Omega}\|_2^2.
\end{align}
The theorem now follows from three lemmas, which control the quantities in the above inequalities. 
The first lemma is identical to the one in Krishnamurthy and Singh~\cite{krishnamurthy2013low} while the third is from Balzano \emph{et al.}~\cite{balzano2010high}.
The second one improves on both of the similar results from those to works.
\begin{lemma}
With the same notations as in Theorem~\ref{thm:subspace_proj}, with probability $\ge 1 - 2\delta$:
\begin{align}
(1-\alpha)\frac{m}{d}\|v\|_2^2 \le \|v_\Omega\|_2^2 \le (1+\alpha) \frac{m}{d}\|v\|_2^2
\end{align}
\label{lem:alpha}
\end{lemma}
\begin{proof}
The proof is an application of Bernstein's inequality (Theorem~\ref{prop:scalar_bernstein}).
Let $\Omega(i)$ denote the $i$th coordinate in the sample and let $X_i = v_{\Omega(i)}^2 - \frac{1}{d}\|v\|_2^2$ so that $\sum_{i=1}^m X_i = \|v_\Omega\|_2^2 - \frac{m}{d}\|v\|_2^2$.
The variance and absolute bounds are:
\[
\sigma^2 = \sum_{i=1}^m \EE X_i^2 \le \frac{m}{n}\sum_{i=1}^n v_i^4 \le \frac{m}{n}\|v\|_{\infty}^2 \|v\|_2^2, \qquad R = \max \|X_i\| \le \|v\|_{\infty}^2.
\]
Bernstein's Inequality then shows that:
\[
\PP\left( \left| \sum_{i=1}^m X_i\right| \ge t \right) \le 2 \exp\left(\frac{-t^2}{2 \|v\|_{\infty}^2 (\frac{m}{d}\|v\|_2^2 + \frac{1}{3}t}\right).
\]
Setting $t = \alpha \frac{m}{d}\|v\|_2^2$ and using the definition $\mu(v) = d \|v\|_{\infty}^2/\|v\|_2^2$ this bound becomes:
\[
\PP\left( \left| \sum_{i=1}^m X_i\right| \ge \alpha\frac{m}{d}\|v\|_2^2 \right) \le 2 \exp\left(\frac{-\alpha^2}{2\mu(v) (1+\alpha/3)}\right)
\]
And plugging in the definition of $\alpha$ ensures that the probability is upper bounded by $2\delta$.
\end{proof}

\begin{lemma}
With the same notation as Theorem~\ref{thm:subspace_proj} and provided that $m \ge 4 \mu(v) \log(1/\delta)$, with probability at least $1-\delta$:
\begin{align}
\|U_{\Omega}^Tv_\Omega\|_2^2 \le \beta \frac{m}{d}\frac{r \mu(U)}{d}\|v\|_2^2
\end{align}
\label{lem:beta}
\end{lemma}
\begin{proof}
The proof is an application of the vector version of Bernstein's inequality (Proposition~\ref{prop:vector_bernstein}. 
Let $u_i \in \RR^r$ denote the $i$th row of an orthonormal basis for $U$ and set $X_i = u_{\Omega(i)}v_{\Omega(i)}$. 
Since $v \in U^\perp$, the $X_i$s are centered so we are left to compute the variance:
\begin{align*}
\sum_{i=1}^m \EE \|X_i\|^2 = \frac{m}{d}\sum_{j=1}^d\|u_j v_j\|^2 \le \frac{m}{d}\frac{r \mu(U)}{d}\|v\|_2^2 = V
\end{align*}
Applying Proposition~\ref{prop:vector_bernstein} and re-arranging, we have that with probability at least $1-\delta$:
\[
\|U_{\Omega}^Tv_{\Omega}\|_2 \le \sqrt{V} + \sqrt{4 V \log(1/\delta)} = \sqrt{\frac{m}{d} \frac{r \mu}{d}}\|v\|_2\left(1 + 2 \sqrt{\log(1/\delta)}\right)
\]
As long as:
\[
t = \sqrt{4 V \log(1/\delta)} \le V (\max_i \|X_i\|)^{-1}
\]
Since $\max_i \|X_i\| \le \|v\|_{\infty} \sqrt{r \mu/d}$ and using the incoherence assumption on $v$ this condition translates to $m \ge 4 \mu(v) \log(1/\delta)$.
Squaring the above inequality proves the lemma. 
\end{proof}

\begin{lemma}[\cite{balzano2010high}]
Let $\delta > 0$ and $m \ge \frac{8}{3}r \mu(U)\log(2r/\delta)$. Then
\begin{align}
\|(U_{\Omega}^TU_\Omega)^{-1}\|_2 \le \frac{d}{(1-\gamma)m}
\end{align}
with probability at least $1-\delta$ provided that $\gamma < 1$. 
In particular $U_\Omega^TU_\Omega$ is invertible.
\label{lem:gamma}
\end{lemma}

\section{Proof of Lemma~\ref{lem:mb_app}}

Under the uniform at random sampling model, we will apply the non-commutative Bernstein inequality (Proposition~\ref{prop:rect_mb}) to bound $\|\tilde{X} - X\|_2$. 
Recall that for each column $x_t$, we observe a set of $m_{2,t} = m_2 n \frac{\hat{c_t}}{\hat{f}}$ observations and form the zero-filled vector $y_t$ define dby:
\[
y_t = \frac{d}{m_{2,t}}\sum_{s=1}^{m_{2,t}} x_t(i_s) e_{i_s} 
\]
where $\{i_s\}_{s=1}^{m_{2,t}}$ are the observations.
Since the set of observations is sampled with replacement (although duplicates in each half of the sample are thrown out), each entry of $y_t$ occurs with probability $d/m_{2,t}$, so $y_t$ is an unbiased estimate of $x_t$.
So we will apply the rectangular Matrix Bernstein inequality to $y_te_t^T - x_te_t^T$. 
Moreover:
\[
\|y_te_t^T - x_te_t^T\| \le \|y_t\|\|e_t\| + \|x_t\| \le \left(1 + \sqrt{\frac{d \mu}{m_{2,t}}}\right)\|x_t\| \le 2\sqrt{\frac{d \mu}{m_{2,t}}}\|x_t\|
\]
which follows by the triangle inequality, Cauchy-Schwarz and the chain of inequalities:
\[
\|y_t\|_2 \le \sqrt{m_{2,t}}\|y_t\|_{\infty} \le \frac{d}{\sqrt{m_{2,t}}}\|x_t\|_{\infty} \le \sqrt{\frac{d \mu}{m_{2,t}}} \|x_t\|_2
\]
When we plug in for $m_{2,t}$ we get:
\[
\|y_te_t^T - x_te_t^T\| \le 2 \sqrt{\frac{d\mu}{m_2 n} \frac{c_t}{\hat{c}_t} \hat{f}} \le 2 \|X\|_F \sqrt{\frac{d \mu}{m_2 n} \frac{1+\alpha}{1-\alpha}}
\]
where $\alpha$ is the error bound from the first phase of the algorithm. 

As for the variance terms in Proposition~\ref{prop:rect_mb}, both turn out to be quite small as we will soon see.
For the first term:
\begin{align*}
\| \sum_{t=1}^n \EE e_ty_t^Ty_te_t^T - e_tx_t^Tx_te_t^T\| &= \| \sum_{t=1}^n e_te_t^T (\EE \|y_t\|^2 - \|x_t\|^2)\| = \\
&= \| \sum_{t=1}^n e_te_t^T (\frac{d}{m_{2,t}}-1) \|x_t\|^2\| \le 2d \max_{t \in [n]} \frac{\|x_t\|^2}{m_{2,t}}
\end{align*}
The first equality is straightforward while the second follows from linearity of expectation and the fact that each coordinate of $y_t$ is non-zero with probability $m_{2,t}/d$. 
The third line follows from the fact that applying the sum leads to an $n \times n$  diagonal matrix with $\frac{d}{m_{2,t}}\|x_t\|^2$ on the diagonal.
When we use our definition of $m_{2,t}$ this becomes:
\[
\| \sum_{t=1}^n \EE e_ty_t^Ty_te_t^T \| \le \frac{2d}{m_2 n} \|X\|_F^2\frac{1+\alpha}{1-\alpha}
\]

For the second term, we have:
\begin{align*}
\| \sum_{t=1}^n \EE y_te_t^Te_t y_t^T- \EE x_te_t^Te_tx_t^T\| &= \|\sum_{t=1}^n \EE y_ty_t^T- x_tx_t^T\| = \|\sum_{t=1}^n (\frac{d}{m_{2,t}}-1) \textrm{diag}(x_t(1)^2, \ldots, x_t(d)^2)\|\\
& \le \max_{i \in [d]} \sum_{t=1}^n \frac{2d}{m_{2,t}} x_t(i)^2 \le \sum_{i=1}^n \frac{2\mu}{m_{2,t}} \|x_t\|_2^2 \le \|X\|_F^2\frac{2\mu}{m_2}\frac{1+\alpha}{1-\alpha}
\end{align*}
Here the first equality is trivial while the second one uses the fact that off diagonals of $y_ty_t^T$ are unbiased for $x_tx_t^T$ and hence we are left with a diagonal matrix.
To arrive at the second line we note that the spectral norm a diagonal matrix is simply the largest diagonal entry.
Then we apply the incoherence assumption and final our sampling distribution. 

At this point we may apply the inequality which states that with probability $\ge 1-\delta$:
\begin{align*}
\|\sum_{t=1}^n y_te_t^T - x_te_t^T\| \le \|X\|_F\sqrt{\frac{1+\alpha}{1-\alpha}} \left( \sqrt{\frac{4}{m_2} \max(\frac{d}{n}, \mu)\log(\frac{d+n}{\delta})} + \frac{4}{3}\sqrt{\frac{d \mu}{m_2 n}}\log(\frac{d+n}{\delta})\right)
\end{align*}

\section{Some Concentration Inequalities}

Here we collect a number of concentration inequalities used in our proofs.
\begin{proposition}[Scalar Bernstein]
Let $X_1, \ldots, X_n$ be independent, centered scalar random variables with $\sigma^2 = \sum_{i=1}^n \EE [X_i^2]$ and $R = \max_{i} |X_i|$.
Then:
\begin{align}
\PP\left( \sum_{i=1}^n X_i \ge t\right) \le \exp\left\{\frac{-t^2}{2\sigma^2 + \frac{2}{3}Rt}\right\}
\end{align}
\label{prop:scalar_bernstein}
\end{proposition}

\begin{proposition}[Vector Bernstein~\cite{gross2011recovering}]
Let $X_1, \ldots, X_n$ be independent centered random vectors with $\sum_{i=1}^n \EE \|X_i\|_2^2 \le V$. 
Then for any $t \le V (\max_{i} \|X_i\|_2)^{-1}$:
\begin{align}
\PP\left( \left\|\sum_{i=1}^n X_i\right\|_2 \ge \sqrt{V} + t\right) \le \exp\left\{ \frac{-t^2}{4V}\right\}
\end{align}
\label{prop:vector_bernstein}
\end{proposition}

\begin{proposition}[Matrix Bernstein~\cite{tropp2011user}]
Let $X_1, \ldots, X_n$ be independent, random, self-adjoint matrices with dimension $d$ satisfying:
\[
\EE X_k = 0 \qquad \textrm{and} \qquad \|X_k\|_2 \le R \textrm{ almost surely}.
\]
Then, for all $t \ge 0$,
\[
\PP\left(\left\|\sum_{k=1}^nX_k\right\| \ge t\right) \le d \exp\left( \frac{-t^2/2}{\sigma^2 + Rt/3}\right) \qquad \textrm{where} \qquad \sigma^2= \left\|\sum_{k=1}^n \EE X_k^2\right\|
\]
\label{prop:matrix_bernstein}
\end{proposition}

\begin{proposition}[Rectangular Matrix Bernstein~\cite{tropp2011user}]
Let $X_1, \ldots, X_n$ be independent random matrices with dimension $d_1 \times d_2$ satisfying:
\[
\EE X_k = 0 \qquad \textrm{and} \qquad \|X_k\|_2 \le R \textrm{ almost surely}.
\]
Define:
\[
\sigma^2 = \max\left\{\left\|\sum_{k=1}^n \EE(X_kX_k^T)\right\|_2, \left\|\sum_{k=1}^n\EE(X_k^TX_k)\right\|_2\right\}.
\]
Then, for all $t \ge 0$,
\[
\PP\left( \left\| \sum_{k=1}^n X_k\right\|_2 \ge t\right) \le (d_1 + d_2) \exp\left( \frac{-t^2/2}{\sigma^2 + Rt/3}\right).
\]
\label{prop:rect_mb}
\end{proposition}

\end{document}